\documentclass[11pt]{amsart}

\usepackage{lipsum}
\usepackage{amsfonts}
\usepackage{graphicx}

\usepackage{amssymb,amsmath}
\usepackage{booktabs}
\usepackage{dsfont}
\usepackage[colorlinks, citecolor=red]{hyperref}
\usepackage{color}

\usepackage{mathrsfs}

\usepackage{subfigure}

\usepackage{caption}
\usepackage{bm}
\usepackage{algorithm}
\usepackage{algorithmicx}
\usepackage[noend]{algpseudocode}

\newcommand{\xkh}[1]{\left(#1\right)}
\newcommand{\axkh}[1]{(#1)}

\newcommand{\dkh}[1]{\left\{#1\right\}}
\newcommand{\zkh}[1]{\left[#1\right]}

\newcommand{\normf}[1]{\left\|{#1}\right\|_F}

\newcommand{\norms}[1]{\left\|{#1}\right\|}

\newcommand{\ucnormf}[1]{\|{#1}\|_F}
\newcommand{\ucnorms}[1]{\|{#1}\|}

\newcommand{\abs}[1]{\left\lvert#1\right\rvert}

\newcommand{\argmin}[1]{{\rm arg\;}\mathop{\rm min}\limits_{#1}}

\newcommand{\minm}[1]{\mathop{\rm min}\limits_{#1}}
\newcommand{\maxm}[1]{\mathop{\rm max}\limits_{#1}}

\newcommand{\sumin}{\sum_{i=1}^n}

\newcommand{\dist}[1]{\mathrm {dist}\left(#1\right)}
\newcommand{\df}[1]{{\mathrm d}_F\left(#1\right)}
\newcommand{\ucdf}[1]{{\mathrm d}_F (#1 )}

\newcommand{\odgroup}{\mathrm{O}(d)}

\newcommand{\sedgroup}{\mathrm{SE}(d)}
\newcommand{\odn}{{\mathrm O}(d)^{\otimes n}}

\newcommand{\PP}{{\mathbb P}}

\newcommand{\N}{{\mathbb N}}

\newcommand{\R}{{\mathbb R}}

\newcommand{\lkh}{{(l)}}
\newcommand{\T}{\top}

\newcommand{\ma}{{\bm A}}

\newcommand{\md}{{\bm D}}
\newcommand{\mf}{{\bm F}}
\newcommand{\mg}{{\bm G}}
\newcommand{\mh}{{\bm H}}
\newcommand{\mi}{{\bm I}}

\newcommand{\mr}{{\bm R}}
\newcommand{\ms}{{\bm S}}
\newcommand{\muu}{{\bm U}}
\newcommand{\mv}{{\bm V}}
\newcommand{\mx}{{\bm X}}
\newcommand{\my}{{\bm Y}}
\newcommand{\mz}{{\bm Z}}
\newcommand{\mzero}{{\bm 0}}

\newcommand{\midd}{{\bm I}_d}

\newcommand{\myu}{\tilde{\bm Y}}

\newcommand{\mxu}{\tilde{\bm X}}

\newcommand{\mxt}{{\bm X}^t}
\newcommand{\mxtt}{{\bm X}^{t+1}}
\newcommand{\mxttl}{{\bm X}^{t+1,(l)}}
\newcommand{\mxtu}{\tilde{\bm X}^t}
\newcommand{\mxttu}{\tilde{\bm X}^{t+1}}

\newcommand{\mxit}{{\bm X}_i^t}
\newcommand{\mxitu}{\tilde{\bm X}_i^t}

\newcommand{\mxitt}{{\bm X}_i^{t+1}}
\newcommand{\mxittu}{\tilde{\bm X}_i^{t+1}}

\newcommand{\mxtl}{{\bm X}^{t,(l)}}

\newcommand{\mxitl}{{\bm X}_i^{t,(l)}}

\newcommand{\mxltl}{{\bm X}_l^{t,(l)}}

\newcommand{\mxittl}{{\bm X}_i^{t+1,(l)}}

\newcommand{\mft}{{\bm F}^t}
\newcommand{\mfit}{{\bm F}_i^t}

\newcommand{\mftl}{{\bm F}^{t,(l)}}
\newcommand{\mfitl}{{\bm F}_i^{t,(l)}}

\newcommand{\mgt}{{\bm G}^t}
\newcommand{\mgit}{{\bm G}_i^t}

\newcommand{\mgtl}{{\bm G}^{t,(l)}}
\newcommand{\mgitl}{{\bm G}_i^{t,(l)}}

\newcommand{\mq}{{\bm Q}}
\newcommand{\mqu}{\tilde{\bm Q}}
\newcommand{\mqt}{{\bm Q}^t}
\newcommand{\mqtt}{{\bm Q}^{t+1}}

\newcommand{\mqtl}{{\bm Q}^{t,(l)}}

\newcommand{\mw}{{\bm W}}
\newcommand{\mwl}{{\bm W}_l}
\newcommand{\mwll}{{\bm W}^{(l)}}

\newcommand{\oi}{{\mathcal I}}
\newcommand{\op}{\mathrm{sgn}}
\newcommand{\opn}{\mathrm{sgn}}

\newcommand{\optt}{{\mathcal P}_{T_{\mx^t}}}
\newcommand{\optit}{{\mathcal P}_{T_{\mx_i^t}}}

\newcommand{\opttl}{{\mathcal P}_{T_{\mxtl}}}
\newcommand{\optitl}{{\mathcal P}_{T_{\mxitl}}}

\newcommand{\os}{{\mathcal S}}
\newcommand{\osn}{{\mathcal S}}

\newcommand{\osgn}{\mathrm{sgn}}

\newcommand{\opttu}{{\mathcal P}_{T_{\mxtu}}}
\newcommand{\optitu}{{\mathcal P}_{T_{\mxitu}}}

\newcommand{\ep}{\varepsilon}

\newcommand{\tr}{{\rm tr}}

\renewcommand{\omega}{\eta}

\newcommand{\sqrttwo}{\sqrt{2}}

\newcommand{\sqrtn}{\sqrt{n}}
\newcommand{\sqrtd}{\sqrt{d}}
\newcommand{\sqrtnd}{\sqrt{nd}}

\newcommand{\sqrtln}{\sqrt{\log n}}

\newcommand{\et}{e^t}

\newcommand{\etf}{e^t_F}
\newcommand{\ettf}{e^{t+1}_F}

\newcommand{\ebar}{\bar{e}}
\newcommand{\ebarf}{\bar{e}_F}

\newcommand{\RNum}[1]{\uppercase\expandafter{\romannumeral #1\relax}}

\newtheorem{definition}{Definition}[section]

\newtheorem{theorem}[definition]{Theorem}
\newtheorem{lemma}[definition]{Lemma}

\newtheorem{remark}[definition]{Remark}

\date{}

\begin{document}

\author{Haiyang Peng}
\address{School of Mathematical Sciences, Beihang University, Beijing, 100191, China }
\email{haiyangpeng@buaa.edu.cn}

\author{Deren Han}
\address{School of Mathematical Sciences, Beihang University, Beijing, 100191, China }
\email{handr@buaa.edu.cn}

\author{Xin Chen}
\address{Department of Systems Engineering and Engineering Management, The Chinese University of Hong Kong, HKSAR, 999077, China }
\email{xinchen@cuhk.edu.hk}


\author{Meng Huang}
\address{School of Mathematical Sciences, Beihang University, Beijing, 100191, China }
\email{menghuang@buaa.edu.cn}


\baselineskip 18pt

\bibliographystyle{plain}

\title[NS-RGS: Newton-Schulz Based Riemannian Gradient Method]{NS-RGS: Newton-Schulz Based Riemannian Gradient Method for Orthogonal Group Synchronization}

\begin{abstract}
Group synchronization is a fundamental task involving the recovery of group elements from pairwise measurements. 
For orthogonal group synchronization, the most common approach reformulates the problem as a constrained nonconvex optimization and solves it using projection-based methods, such as the generalized power method. 
However, these methods rely on exact SVD or QR decompositions in each iteration, which are computationally expensive and become a bottleneck for large-scale problems. 
In this paper, we propose a Newton-Schulz-based Riemannian Gradient Scheme (NS-RGS) for orthogonal group synchronization that significantly reduces computational cost by replacing the SVD or QR step with the Newton-Schulz iteration. This approach leverages efficient matrix multiplications and aligns perfectly with modern GPU/TPU architectures. 
By employing a refined leave-one-out analysis, we overcome the challenge arising from statistical dependencies, and establish that NS-RGS with spectral initialization  achieves linear convergence to the target solution up to near-optimal statistical noise levels. 
Experiments on synthetic data and real-world global alignment tasks demonstrate that NS-RGS attains accuracy comparable to state-of-the-art methods such as the generalized power method, while achieving nearly a 2\(\times\) speedup.
\end{abstract}

\maketitle

\keywords{Keywords: group synchronization, orthogonal group, nonconvex optimization, Riemannian optimization, leave-one-out analysis}


\section{Introduction}
\subsection{Group synchronization}
The recovery of group elements from pairwise measurements constitutes one of the most pervasive and fundamental challenges in high-dimensional data analysis, scientific computing, and machine perception \cite{chatterjee2017robust, chensijin2024non, lingshuyang2022improved, shkolnisky2012viewing}. 
In this paper, we study the  \textit{orthogonal group synchronization} problem, which seeks to recover a collection of  orthogonal group elements \(\dkh{\mz_{i}}_{i=1}^{n}\) from their corrupted pairwise measurements, 
\begin{align*}
	\ma_{ij} = \mz_{i} \mz_{j}^\T + \sigma \mw_{ij} \in \R^{d \times d}, \qquad 1 \le i \neq j \le n,
\end{align*}
where \(\mw_{ij} = \mw_{ji}^\T \in \R^{d \times d}\) are Gaussian random matrices with independent standard normal entries, $\sigma$ denotes the noise level, and
\[
\mz_{i} \in \odgroup = \{\mg \in \R^{d \times d} : \mg \mg^\T = \mg^\T \mg = \midd\}.
\]
This problem is fundamental across science and engineering, underpinning tasks like computer vision \cite{chatterjee2017robust, huang2013consistent, ozyecsil2017survey}, robotics \cite{chenxin2025non, rosen2019se}, and cryo-EM imaging \cite{shkolnisky2012viewing, singer2018mathematics, singer2011viewing}. 

Based on the least squares criterion, one can employ the following program to estimate the groud truth \(\dkh{\mz_{i}}_{i=1}^{n}\): 
\begin{equation}\label{def:f}
	\minm{\mx_i \in \odgroup} F(\mx) = \sum_{i \ne j}^{n} \frac 12 \ucnormf{\mx_i \mx_j^\T - \ma_{ij}}^2 . 
\end{equation} 
While generalized power method (GPM) \cite{chensijin2024non, lingshuyang2022improved, liuhuikang2023unified} and Riemannian gradient-based algorithms \cite{boumal2023introduction, rosen2019se} are at the forefront of group synchronization research, the retraction operation remains a significant computational bottleneck. On the orthogonal group $\odgroup$ or Stiefel manifold, this step requires per-node SVD or polar decomposition at every iteration. Such operations scale poorly on hardware accelerators like GPUs or TPUs. Although these platforms are highly efficient at dense matrix multiplication, they are frequently bottlenecked by the sequential nature of complex factorizations. 

To circumvent the cost of matrix factorizations, this study develops the Newton-Schulz based Riemannian gradient method for Orthogonal Group Synchronization (NS-RGS). The core innovation lies in employing Newton-Schulz iterations \cite{kenney1991rational, nicholas2008functions} to approximate the orthogonal projection, effectively replacing SVD retractions with highly parallelizable matrix multiplications. This algorithmic design reflects a deliberate trade-off between local precision and global computational efficiency. Furthermore, we establish that exact manifold feasibility is not strictly necessary for the convergence of the Riemannian gradient flow, provided the iterative approximation error is sufficiently controlled.

\subsection{Related work}
Early rigorous group synchronization methods addressed nonconvex constraints via spectral methods \cite{bandeira2013cheeger, cucuringu2012sensor, singer2011angular}. 
Spectral methods are now predominantly viewed as initialization tools for finding the basin of attraction of iterative solvers, as they perform only a single denoising step without enforcing group constraints. 
To achieve statistical optimality while maintaining a convex landscape, researchers have extensively leveraged semidefinite programming (SDP) relaxations \cite{arie2012global, bandeira2018random, bandeira2017tightness, singer2011angular}. 
However, the computational complexity of solving an $n \times n$ SDP using standard interior-point solvers scales as $O(n^{3.5})$ \cite{chensijin2024non,liuhuikang2023unified}, rendering it prohibitive for large-scale applications. Consequently, subsequent research has shifted focus toward more efficient alternatives, such as GPM \cite{chensijin2024non, liu2017estimation, liuhuikang2023unified, zhu2021orthogonal} and Riemannian gradient-based algorithms \cite{boumal2023introduction, liuhuikang2023resync, rosen2019se, zhulinglingzhi2023rotation}. 

Subsequent scholars examined nonconvex landscape by the Burer-Monteiro factorization, by relaxing an orthogonal matrix to an element on Stiefel manifold 
\[\mbox{St}(p,d) := \dkh{\ms_i \in \R^{d \times p} : \ms_i \ms_i^\T = \midd}. \]
Boumal \cite{boumal2016nonconvex} demonstrated that, within the framework of the stochastic group block model (SGBM), i.e., \(d = 1\), the second-order necessary conditions for optimality are sufficient for global optimality with \(\sigma \lesssim O(n^{1/6})\). 
For the case \(d \ge 2\), \cite{lingshuyang2023soving} provided a general sufficient condition for \(p \ge 2 d+1\) that guarantees the absence of spurious local minimizers in the optimization landscape with \(\sigma \lesssim O(n^{1/4})\). 
Recently, Ling enhanced this result, generalizing the benign landscape analysis to higher noise level \(\sigma \lesssim O(\sqrtn/d)\) (modulo logarithmic terms) \cite{lingshuyang2025local}.


Although landscape analysis suggesting local minima are not critical bottlenecks, proving algorithmic convergence and its rate is challenging due to the inherent dependency between iterates and the random noise. 
To address this, Zhong and Boumal \cite{zhongyiqiao2018near} introduced the leave-one-out technique to phase synchronization, providing a near-optimal guarantee on the tightness of the solution under the regime $\sigma \lesssim \sqrt{n/\log n}$. 
Subsequently, Ling \cite{lingshuyang2022improved} extended this framework to $\odgroup$ synchronization, establishing near-optimal noise thresholds of \(\sigma \lesssim O(\sqrtn/d)\) for the GPM. 
By constructing a sequence of auxiliary iterates, the leave-one-out analysis serves as a powerful mechanism for decoupling the statistical dependencies between signals and noise. 
It is widely employed in problems including phase retrieval, matrix completion, and blind deconvolution to ensure theoretical convergence and statistical independence \cite{abbe2020entrywise, chenyuxin2019, chenyuxin2020, huangmeng2022, macong2020, penghaiyang2024npr}. 
Furthermore, Gao and Zhang \cite{gaochao2023optimal} derived the exact minimax risk, demonstrating that the GPM is not only computationally superior to SDP relaxation methods but also statistically equivalent.

\subsection{Our contributions}\label{sec:contributions}
Due to the significant computational cost of the retraction step with SVDs, this paper employs the Newton-Schulz iterations 
\[\ms_{t+1} = \frac 12 \ms_{t} \Big( 3\midd - \ms_{t}^\T \ms_t \Big), \quad t = 0, 1, \cdots, \]
to approximate the matrix sign function for retraction. 
Newton-Schulz iterations obviate the computational bottlenecks inherent in SVD-based approaches by employing only matrix multiplications, thereby attaining peak floating-point operations (FLOPs) utilization on hardware accelerators like tensor cores. 


For the group synchronization problem investigated in this paper, we develop an improved Riemannian gradient descent framework based on the Newton-Schulz iteration. Our method eliminates the need for costly matrix factorizations and supports efficient parallel execution, leading to lower time complexity per iteration. 
The contributions of this work encompass algorithmic innovation, rigorous theoretical analysis, and extensive empirical validation, effectively bridging the divide between high-dimensional statistics and high-performance computing. Our primary contributions are three-fold: 
\begin{itemize}
	\item An efficient algorithmic framework: We propose NS-RGS, a novel framework for group synchronization designed to circumvent the computational bottleneck in conventional methods. Rather than performing exact but computationally intensive SVDs for retractions at each iteration, we leverage the Newton-Schulz approximation. 
	By substituting computationally demanding decomposition tasks with straightforward matrix multiplications, this alternative approach achieves significantly enhanced efficiency with a negligible loss in accuracy. 

	\item Rigorous theoretical guarantees via Leave-one-out analysis: We establish a solid theoretical foundation for our inexact Riemannian optimization framework. 
	Traditional convergence analysis is not applicable in this context, as the iterates become statistically coupled with the noise. 
	To decouple these complicated dependencies, we leverage a sophisticated leave-one-out analysis, constructing auxiliary iterates that maintain independence from specific noise entries. 
	This approach enables us to rigorously demonstrate that our algorithm achieves linear convergence and reaches near-optimal statistical noise limits. 
	Whereas Liu et al. \cite{liuhuikang2023unified} focused on the global convergence of the iterate \(\mxt\) toward the ground truth \(\mz\), our work provides a more refined analysis that guarantees the proximity of individual components. 
	Importantly, we show that each component \(\mxit\) tightly nears an optimal rotation \(\mqt\). 

	\item Experimental validation: To validate the efficiency of NS-RGS, we conducted extensive evaluations on synthetic data and real-world 3D global alignment problem (e.g., the Stanford Lucy dataset). The results indicate that employing a single-step Newton-Schulz iteration yields accuracy on par with leading methods like GPM and Riemannian trust-region method (RTR) \cite{boumal2015riemannian}. More importantly, NS-RGS offers a substantial reduction in computational costs, achieving up to a \(2.3 \times\) acceleration in convergence compared to the standard GPM for real-world task. 
\end{itemize}

\subsection{Notations}
For a given matrix \(\mx\), \(\norms{\mx}\), \(\normf{\mx}\), and \(\norms{\mx}_*\) represents the operator norm, the Frobenius norm, and the nuclear norm. Then \(\mx^\T\) stands for the transpose of $\mx$. We denote the smallest and the \(i\)-th largest singular value (and eigenvalue) of \(\mx\) by \(\sigma_{\min} (\mx)\) and \(\sigma_{i} (\mx)\) (and \(\lambda_{\min} (\mx)\) and \(\lambda_{i} (\mx)\)). 
\(\odgroup = \{\mg \in \R^{d \times d} : \mg \mg^\T = \mg^\T \mg = \midd\}\) denotes the orthogonal group. 
The notation \(\mr \in \odn\) means \(\mr\) is a matrix of size \(n d \times d\) whose \(i\)-th block is \(\mr_i \in \odgroup\). 
The notation \([n]\) refers to the set \(\dkh{1, \cdots, n}\). 
Additionally, the notation \(f(m,n)=O(g(m,n))\) or \(f(m,n)\lesssim g(m,n)\) denotes the existence of a constant \(c>0\) such that \( f(m,n)\le cg(m,n) \), while \(f(m,n)\gtrsim g(m,n)\) indicates that there exist a constant \(c>0\) such that \(f(m,n)\ge cg(m,n)\). 

Noting that \(\mx \mx^\T = \mx \mq (\mx \mq)^\T\) for any matrix \(\mx \in \R^{n d \times d}\) and \(\mq \in \odgroup\), we define the distance up to a global rotation 
\[\df{\mx, \my} = \minm{\mq \in \odgroup} \ucnormf{\mx - \my \mq}, \] 
for any \(\mx, \my \in \R^{n d \times d}\).

\section{Algorithm and interpretation}\label{sec:preliminaries}
\subsection{The Newton-Schulz iteration} 
A widely used retraction operator for orthogonal group synchronization  is the matrix sign function. 
Given the full SVD of \(\ma \in \R^{d \times d}\) as $\ma = \muu {\bm \Sigma} \mv^\T$, 
the matrix sign function is denoted as $\op(\ma) = \muu  \mv^\T$. For a block matrix \({\bm \Phi} = (\ma_1; \cdots; \ma_n) \in \R^{nd \times d}\), we also define \(\opn({\bm \Phi}) = \xkh{\op(\ma_1); \cdots; \op(\ma_n)} \in \odn\).



In this paper, we employ the Newton-Schulz iterations to approximate \(\op(\ma)\) for a full-rank matrix \(\ma\). The iterations are given by 
\begin{align}\label{eq:newton-schulz}
	\ms_{t+1} = \frac 12 \ms_{t} \Big( 3\midd - \ms_{t}^\T \ms_t \Big), \quad \ms_0 = \ma,
\end{align}
as outlined in Algorithm \ref{alg:1}. 
Nakatsukasa and Higham \cite{nakatsukasa2012backward} established its backward stability under specific scaling conditions. 
In recent years, the Newton-Schulz iteration has undergone a significant renaissance within the deep learning community. 
Recent work integrates this method into large-scale training frameworks, notably the Muon optimizer \cite{bernstein2024modular, jordan2024muon}. 
Motivated by the significant empirical success of Muon's spectral updates, a large number of theoretical studies have emerged over the past year, aiming to clarify the mechanisms behind its effectiveness from different perspectives \cite{davis2025spectral, kovalev2025understanding,li2025note, chen2025muon, shen2025convergence}.

Lemma \ref{le:ns:bound} establishes the quadratic convergence of \(\ucnorms{\midd - \ms_t^\T \ms_t}\) and indicates that the error \(\ucnorms{\ms_t - \op (\ma)}\) is controlled by \(\ucnorms{\midd - \ms_t^\T \ms_t}\). 
So it follows that the Newton-Schulz iterates achieve quadratic convergence to \(\op (\ma)\). 
From a computational perspective, this implies that a very small number of iterations suffice for the retraction step, thereby significantly enhancing computational efficiency.

\begin{lemma}\cite{kenney1991rational, nicholas2008functions}\label{le:ns:bound}
	Suppose \(\ucnorms{\midd - \ma^\T \ma} < 1\) for \(\ma \in \R^{d \times d}\). Let \(\dkh{\ms_t}\) are generated by the iteration \eqref{eq:newton-schulz}, 
	then it holds \(\ms_t \rightarrow \op (\ma)\) as \(t \rightarrow \infty\). Specifically, one has 
	\[\ucnorms{\ms_t - \op (\ma)} \le \ucnorms{\midd - \ms_t^\T \ms_t} \le \ucnorms{\midd - \ma^\T \ma}^{2^t}, \quad \forall\ t \in \N. \]
\end{lemma}

While the analyses in \cite{kenney1991rational, nicholas2008functions} primarily focus on symmetric matrix targets, Lemma \ref{le:ns:bound} can be rigorously proven via the Jordan-Wielandt matrices \cite{stewart1990matrix}. 

\begin{algorithm}
	\caption{The Newton-Schulz method for the matrix sign function}\label{alg:1}
	\begin{algorithmic}[0]
		\State \textbf{Input:} A matrix $\ma \in \R^{d \times d}$, the maximum number of iterations $T_s$.
		\State \textbf{Initialization:} Set $\ms_0 = \ma$. 
		\State \textbf{Iterative updates: for} $t = 0,1,2, \cdots,T_s-1$ \textbf{do} 
        	\begin{align*}
				\ms_{t+1} = \frac 12 \ms_{t} \Big( 3\midd - \ms_{t}^\T \ms_t \Big). 
			\end{align*}
	\end{algorithmic}
\end{algorithm}

\subsection{Riemannian gradient synchronization with inexact iteration}
The program we consider is 
\begin{align}\label{def:f problem}
	\minm{\mx_i \in \odgroup} F(\mx) = \sum_{1 \le i \ne j \le n} \frac 12 \ucnormf{\mx_i \mx_j^\T - \ma_{ij}}^2. 
\end{align}
To solve this nonlinear system problem, we employ a block coordinate descent (BCD) strategy to update $\mx_i$ as follows, 
\begin{align*}
	\mxitt = {\rm arg}\mathop{\rm min}_{\mx_i \in \odgroup} \sum_{1 \le i \ne j \le n} \frac 12 \ucnormf{\mx_i \mx_j^\T - \ma_{ij}}^2 = \argmin{\mx_i} f_i(\mx_1^t, \cdots, \mx_i, \cdots, \mx_n^t), 
\end{align*}
where \(f_i(\mx_1, \cdots, \mx_n) = \sum_{j \ne i}^{n} \frac 12 \ucnormf{\mx_i \mx_j^\T - \ma_{ij}}^2\). 
Note that 
\[\frac{\partial}{\partial \mx_i} f_i(\mxt) = \sum_{j \ne i}^{n} \xkh{\mx_i^t (\mx_j^t)^\T \mx_j^t - \ma_{ij} \mx_j^t} = \sum_{j \ne i}^{n} \xkh{\mx_i^t - \ma_{ij} \mx_j^t}. 
\] 
By leveraging the Riemannian gradient method framework, an iterative scheme is given by 
\begin{equation*}
	\left\{\begin{aligned}
		\mgit &= \sum\nolimits_{j \ne i}^{n} (\mxit - \ma_{ij} \mx_{j}^{t}), \\
		\mfit &= \mxit - \mu \optit (\mgit), \\
		\mxitt &= \osgn (\mfit), 
	\end{aligned}
	\right.
\end{equation*}
for each component \(\mxitt\). 
Here, the tangent space to \(\odgroup\) at \(\mxit \in \odgroup\) is given by \({\rm T}_{\mxit} := \dkh{\mxit \mh:\mh \in \R^{d \times d}, \mh + \mh^\T = 0}\), while the projection of \(\mg \in \R^{d \times d}\) onto this tangent space is \(\optit(\mg) = (\mg - \mxit \mg^\T \mxit)/2\). 

To improve computational efficiency, we employ an inexact retraction
\begin{align*}
	\mxitt = \os \xkh{\mxit - \mu \optit (\mgit)} 
\end{align*}
for \(i \in [n]\) where \(\os(\cdot)\) represent calculating the matrix sign function by the Newton-Schulz method in Algorithm \ref{alg:1}. 
Despite the iteration $\mxitt$ in Algorithm \ref{alg:2}  lies outside the $\odgroup$ manifold, the operator \(\optit (\cdot)\) is still employed as a pseudo projection procedure. 

Due to the non-convexity of problem \eqref{def:f problem}, proper initialization is crucial to ensure convergence toward the global optimum. 
A widely adopted approach is the aforementioned spectral method \cite{bandeira2013cheeger, lingshuyang2022improved, liuhuikang2023resync, singer2011angular}. Specifically, the initial estimate $\mx^0 = \osgn(\my)$, while \(\my\) is derived from the top $d$ eigenvectors of an observation matrix $\ma$. 
This approximation $\mx^0$ ensures the algorithm starts within a favorable basin of attraction. 
Notably, applying a block matrix representation, the measurement in this paper is given by 
\begin{align}\label{def:measurements}
	\ma = \mz \mz^\T + \sigma \mw \in \R^{n d \times n d}, 
\end{align}
where \(\mz = \xkh{\mz_1; \cdots; \mz_n}\) is the ground truth and \(\mw = [\mw_{ij}]_{1 \le i, j \le n}\) is a symmetric random matrix. 
Under the assumption that self-loop information is excluded from the computation, we let $\mw_{ii} = \mzero$ and define $\mw_{ij} \in \R^{d \times d}$ as a Gaussian random matrix with independent standard Gaussian entries for \(1 \le i < j \le n\). 

\begin{algorithm}
	\caption{Riemannian gradient synchronization with inexact iteration}\label{alg:2}
	\begin{algorithmic}[0]
		\State \textbf{Input:} The symmetric observation matrix $\ma = [\ma_{ij}]_{1 \le i, j \le n}$, the maximum number of iterations $T$, the step size $\mu$.
		\State \textbf{Spectral initialization: } 
		Let \(\my \in \R^{nd \times d}\) be the top \(d\) eigenvectors of \(\ma = \mz \mz^\T + \sigma \mw\) with \(\my^\T \my = \midd\). 
		Set $\mx^0 = \opn(\my)$. 
		\State \textbf{Riemannian gradient updates:} 
		\For{$t = 0, 1, 2, \cdots, T-1$} 
    		\For{$i = 1 \text{ to } n$} 
        		\begin{align*}
					\mgit &= \sum\nolimits_{j \ne i}^{n} (\mxit - \ma_{ij} \mx_{j}^{t}), \\
					\mfit &= \mxit - \mu \optit (\mgit), \\
					\mxitt &= \os (\mfit). 
				\end{align*}
    		\EndFor
		\EndFor
	\end{algorithmic}
\end{algorithm}


\section{Main results}\label{sec:results}
In this section, we establish the convergence of Algorithm \ref{alg:2} for orthogonal group synchronization, under the assumption that the noise level $\sigma$ satisfies near-optimal bounds.
 Recall that  the inexact update step in Algorithm \ref{alg:2} is  \(\mxtt = (\mxtt_1;\cdots;\mxtt_n) = \os (\mft)\). 
For convenience, we denote the exact retraction as $\mxttu = (\mxttu_1; \cdots; \mxttu_n) = (\op (\mft_1); \cdots; \op (\mft_n))$, and define the error sequence \(\dkh{\etf}\) by 
\[\etf := \ucnormf{\mxt - \mxtu}\]
for any \(i \in [n]\). Our main result is as follows:


\begin{theorem}\label{th:main}
	Suppose the noise level 
	\[\sigma \le c_0 \frac{\sqrt{nd^{-1}}}{\sqrtd + 10 \sqrtln}\]
	for some constant \(c_0 > 0\). For any fixed \(\mz \in \odn\), assume that the sequence \(\dkh{\mxt}_{t=0}\) from Algorithm \ref{alg:2} satisfies the criterion \(\etf \le \bar{e}_F \le 10^{-4}\). Then with probability at least $1 - \exp (-nd/2) - O(n^{-10})$, Algorithm \ref{alg:2} with step size \(\mu = 1/n\) satisfies 
	\[\ucdf{\mxt, \mz} \le \frac{1}{2^t} \ucdf{\mx^0, \mz} + 56 \bar{e}_F + 8 c_0 \sqrtn\] 
	for \(0 \le t \le t_0 \le n^{10}\). 
\end{theorem}

\begin{remark}
In Theorem \ref{th:main}, we require the inexact error  $\etf=\ucnormf{\mxit - \mxitu}  \le 10^{-4}$. 
According to Lemma \ref{le:ns:bound},  the Newton-Schulz iteration exhibits  quadratic convergence under suitable conditions, which are verified to hold for all iterations (see Lemma \ref{le:error}). 
Consequently,  the inexact error satisfies
	\[\ucnormf{\mxt - \mxtu} \le \sqrtnd \max_{1 \le i \le n} \ucnorms{\mxit - \mxitu} \le 10^{-4},\]
with only \(O \xkh{\log \log \xkh{nd}}\) Newton-Schulz steps  for each \(i \in [n]\), making Algorithm \ref{alg:2} computationally efficient. In practice, our numerical experiments indicate that approximately five steps are sufficient to achieve high accuracy.
\end{remark}

\begin{remark}
In Theorem \ref{th:main}, we adopt a  fixed step size of $1/n$ to simplify the analysis.  In fact, the step size can be relaxed (e.g., to $(0, 2/n)$) with a more intricate proof. 
\end{remark}

\begin{remark}
	It is worth noting that when \(\bar{e}_F = 0\) and \(\sigma \leq c_0 \sqrt{nd^{-1}}/(\sqrtd + 10 \sqrtln)\),  the convergence result of Theorem \ref{th:main}, $\ucdf{\mxt, \mz} \le \frac{1}{2^t} \ucdf{\mx^0, \mz} + 8 c_0 \sqrtn$, aligns with the the best theoretical bounds reported in \cite{lingshuyang2022improved}. 
\end{remark}



\section{Numerical Experiments}\label{sec:experiments}
This section presents a series of numerical experiments conducted on both synthetic and real-world datasets to evaluate the performance of our proposed method. 
All simulations are implemented in MATLAB R2023b and executed with an Intel Core i9-14900HX CPU and 32GB of RAM. 

\subsection{Synthetic data}\label{sec:synthetic data}
In our experiments, we consider the orthogonal group $\odgroup$ with \(d = 25\). 
The ground truth $\mz=(\mz_1; \cdots; \mz_n)^\T$ is constructed by generating i.i.d. standard Gaussian matrices, followed by a projection onto the $\mathrm{O}(d)$ manifold. 
We implement an incomplete observation model with a sampling rate $p \in (0,1]$, as follows 
\begin{align*} 
	\ma_{ij} = \left\{\begin{aligned}
		& \mz_i \mz_j^\T + \sigma \mw_{ij} ,  &&(i,j) \in \mathcal{A},\\
		& \mzero ,                            &&(i,j) \in \mathcal{A}^c,
	\end{aligned}\right. \quad \mbox{with} \ \ (i, j) \in \left\{\begin{aligned}
		& \mathcal{A}   ,  &&\mbox{with ratio}\ p,\\
		& \mathcal{A}^c ,  &&\mbox{with ratio}\ 1-p.
	\end{aligned}\right.
\end{align*}
Note that for \(p = 1\), the observation corresponds to full sampling. 


We evaluate our approach in comparison with the GPM \cite{lingshuyang2022improved} and RTR \cite{boumal2015riemannian}, with the latter utilizing the MATLAB Manopt toolbox to handle subproblems. 
We define the relative error as 
\begin{equation*} 
	\text{Rel. Err.} = \frac{\ucnormf{\mz\mz^\T - \mxt(\mxt)^\T}}{\ucnormf{\mz\mz^\T}}.
\end{equation*}
The stopping criteria for our method and GPM are defined as the relative decrease in the objective function value in \eqref{def:f} as 
\begin{equation*} 
	R^t = \frac{F(\mxt)-F(\mxtt)}{F(\mxtt)}. 
\end{equation*}
Execution terminates when $R^t < 10^{-8}$ or iterations reach $\text{MaxIter} = 100$. 
Additionally, the termination criterion for RTR is that the optimal condition matrix $\ms = {\bm \Lambda} - {\bm A}$ \cite[Proposition 5.1]{lingshuyang2023soving} satisfies $\lambda_{\min} (\ms) \ge -10^{-10}$, which is numerically non-negative. 
We set \(n = 500\), \(\sigma \in \dkh{0.02,0.1,0.2}\), and \(p \in \dkh{0.5,0.8,1.0}\). The step size for our method is set to be \(\mu = 1/np\). 
To ensure a fair comparison, all three approaches are initialized using the same spectral method as a consistent baseline. Table \ref{table:1} summarizes the numerical results, reporting the mean relative error and CPU time averaged over 10 independent trials across various parameter settings. 
By substituting SVD with a single-step Newton-Schulz iteration, we achieve high-precision retractions. Empirical evidence indicates that a single iteration is computationally sufficient, providing a robust balance between approximation error and computational efficiency. 
Our proposed algorithm achieves significantly faster convergence compared to existing methods, providing an approximately 1.7\(\times\) speedup with only a negligible impact on the relative error. 

Generally, as the noise level $\sigma$ increases, the accuracy of the recovery results decreases. 
This behavior is consistent with our theoretical findings. Additionally, we observe that reducing the observation probability $p$, i.e., increased graph sparsity, further degrades algorithmic accuracy. 

\renewcommand{\arraystretch}{1.2}

\begin{table}[ht]
\centering
\caption{Numerical results for varying noise levels and observation rates. }\label{table:1}
\begin{tabular}{l|cc|cc|cc}
\toprule
 \multicolumn{1}{c}{} & \multicolumn{2}{c}{$\sigma = 0.02$} & \multicolumn{2}{c}{$\sigma = 0.1$} & \multicolumn{2}{c}{$\sigma = 0.2$} \\
\cmidrule(lr){2-3} \cmidrule(lr){4-5} \cmidrule(lr){6-7}
Algorithm & Rel. Err. & Time (s) & Rel. Err. & Time (s) & Rel. Err. & Time (s) \\
\midrule
\multicolumn{7}{c}{$n = 500,\ \, p = 1$} \\
\cmidrule(l{4pt}r{4pt}){1-7}
RTR & 4.38E-03 & 8.314 & 2.19E-02 & 8.874 & 4.38E-02 & 9.146 \\
GPM & 4.38E-03 & 0.282 & 2.19E-02 & 0.287 & 4.38E-02 & 0.289 \\
NS-RGS & 4.38E-03 & \bf{0.154} & 2.19E-02 & \bf{0.158} & 4.38E-02 & \bf{0.161} \\
\midrule
\multicolumn{7}{c}{$n = 500,\ \, p = 0.8$} \\
\cmidrule(l{4pt}r{4pt}){1-7}
RTR & 4.90E-03 & 9.076 & 2.45E-02 & 9.339 & 4.91E-02 & 9.387 \\
GPM & 4.90E-03 & 0.292 & 2.45E-02 & 0.302 & 4.91E-02 & 0.423 \\
NS-RGS & 4.90E-03 & \bf{0.158} & 2.45E-02 & \bf{0.163} & 4.91E-02 & \bf{0.234} \\
\midrule
\multicolumn{7}{c}{$n = 500,\ \, p = 0.5$} \\
\cmidrule(l{4pt}r{4pt}){1-7}
RTR & 6.21E-03 & 9.162 & 3.11E-02 & 9.879 & 6.21E-02 & 9.747 \\
GPM & 6.21E-03 & 0.420 & 3.11E-02 & 0.422 & 6.21E-02 & 0.434 \\
NS-RGS & 6.21E-03 & \bf{0.241} & 3.11E-02 & \bf{0.244} & 6.21E-02 & \bf{0.240} \\
\bottomrule
\end{tabular}
\end{table}
\renewcommand{\arraystretch}{1}

\subsection{Real data}
In this experiment, we investigate the global alignment problem of 3D scans using the Lucy dataset from the Stanford 3D Scanning Repository, as in \cite{wanglanhui2013exact}. 
Our experiments utilize a downsampled subset comprising \(n = 168\) scans with \(d = 3\) and approximately four million triangles, from which 2,628 pairwise transformations are derived based on preliminary estimates. Consequently, the sparsity is roughly $p = 0.187$. 
These transformations are perturbed by Gaussian noise with $\sigma = 0.05$. The evaluation metrics, including relative error and stopping criteria, are consistent with the experimental protocols detailed in Section \ref{sec:synthetic data}. 
We also define the mean squared error (MSE) \cite{wanglanhui2013exact} of the estimated rotation matrices \(\hat{\mx}_1, \cdots, \hat{\mx}_n\) as 
\begin{equation*} 
	\text{MSE} := \frac{1}{n} \ucdf{\hat{\mx}, \mx}^2 = \min_{\mq \in \odgroup} \frac{1}{n} \sumin \ucnormf{\hat{\mx}_i - \mx_i \mq}^2.
\end{equation*}

As demonstrated in Table \ref{table:2}, our method yields competitive accuracy compared to the GPM baseline while achieving a roughly $2.3\times$ speedup. 
Figure \ref{fig:mse} further illustrates the MSEs and the unsquared residuals \(\ucnormf{\hat{\mx}_i - \mx_i \mq}\) (\(i = 1, \cdots, 168\)). 
These comparisons indicate that our method achieves a precision level comparable to both the RTR and GPM baselines, with a negligible relative error of less than $0.3\%$. 
This confirms that the Newton-Schulz approximation can maintain the reconstruction performance of exact retraction methods.

\renewcommand{\arraystretch}{1.2}

\begin{table}[ht]
\centering
\caption{Numerical results for Lucy dataset.}\label{table:2}
\begin{tabular}{l|cc}
\toprule
Algorithm & Rel. Err. & Time (s)  \\
\midrule
RTR & 1.52E-02 & 0.207  \\
GPM & 1.52E-02 & 0.055 \\
NS-RGS & 1.52E-02 & \bf{0.023}  \\
\bottomrule
\end{tabular}
\end{table}
\renewcommand{\arraystretch}{1}

\begin{figure}[H] 
	\centering 
	\includegraphics[width=0.5\textwidth]{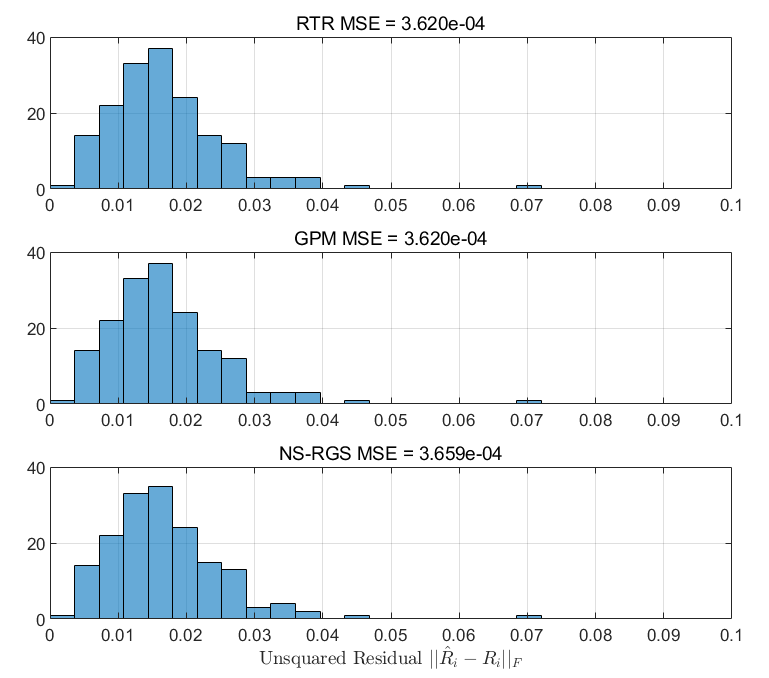} 
	\caption{Histogram of the unsquared residuals of RTR, GPM and NS-RGS for the Lucy dataset.} 
	\label{fig:mse} 
\end{figure}

The reconstruction results are illustrated in Figure \ref{fig:reconstruction}. 
Notably, at a noise level of $\sigma=0.05$, delicate structures such as the torch exhibit ghosting artifacts. This phenomenon stems from the intrinsic limitations of the $\ell_2$-based objective function when subjected to substantial noise, rather than the deficiency in the algorithm itself. To further analyze the error distribution, Figure \ref{fig:heatmap} provides a difference heatmap. 
After performing rigid alignment to account for global rotation, the Euclidean distances between the reconstructed vertices and the ground truth are computed and mapped onto the 3D model, where blue and red represent low and high error magnitudes, respectively. 
The heatmap reveals a predominantly blue surface, suggesting that all three methods maintain high reconstruction quality across most regions.

\begin{figure}[H] 
	\centering 
	\includegraphics[width=1\textwidth]{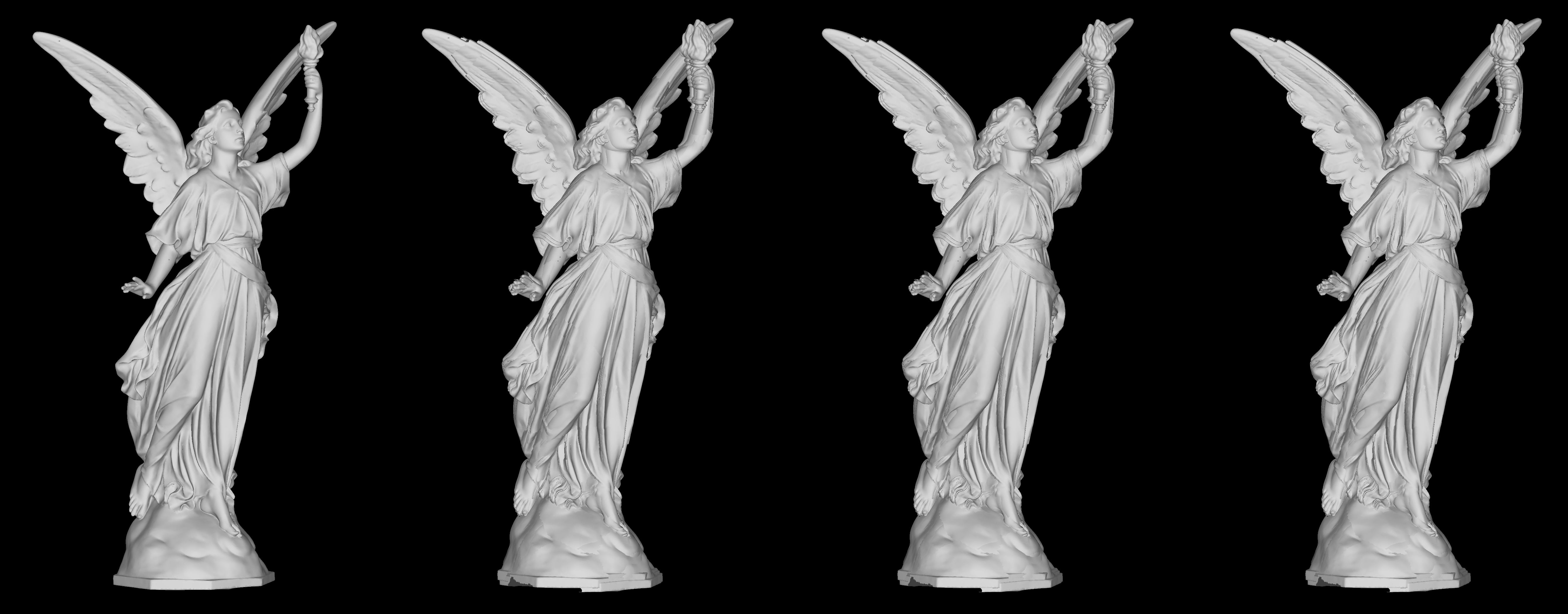} 
	\caption{Views of the reconstructions from the Lucy dataset. The left one is the ground truth. The last three are the reconstruction results of the RTR, GPM, and NS-RGS methods, respectively.} 
	\label{fig:reconstruction} 
\end{figure}

\begin{figure}[H] 
	\centering 
	\includegraphics[width=0.95\textwidth]{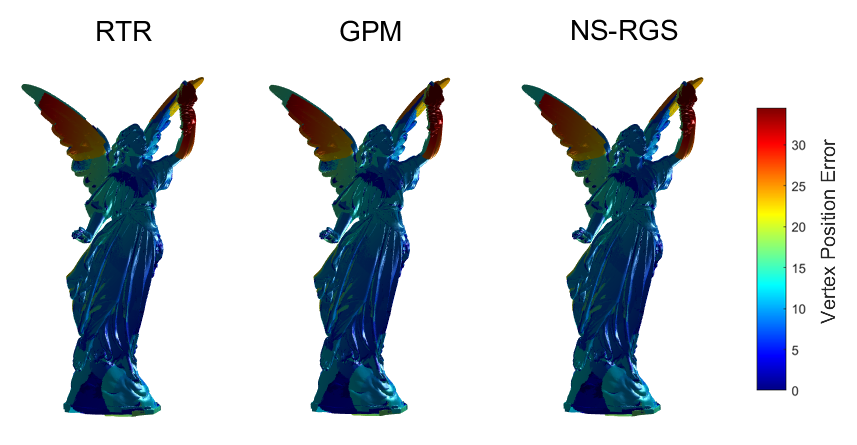} 
	\caption{Difference heatmaps for the Lucy dataset. The entire model is approximately \(930 \times 530 \times 1600\) units.} 
	\label{fig:heatmap} 
\end{figure}

\section{Proofs}
Recall the observation is \(\ma = \mz \mz^\T + \sigma \mw \in \R^{n d \times n d}\) with the ground truth \(\mz = \xkh{\mz_1; \cdots; \mz_n}\). 
Denoting \(\md = {\rm diag} (\mz_1, \cdots, \mz_n)\), 
one can notice that 
\[\md^\T \ma \md = \zkh{\midd + \sigma \mz_i^\T \mw_{ij} \mz_j}_{1 \le i, j \le n}\]
by simple linear operations. 
Combining this argument with the rotation invariance of Gaussian random matrices, throughout this section, one supposes that the ground truth \(\mz = (\midd; \cdots; \midd) \in \odn\) without loss of generality. 

\subsection{Error contraction}
The essence of the argument in this section is the region of incoherence and contraction (RIC), defined as 
\begin{subequations}\label{def:ric}
	\begin{align}
	 	\label{def:ric:1} \ucdf{\mxt, \mz}&\le \ep \sqrtn, \\
		\label{def:ric:2} \max_{1 \le l \le n}\ucnormf{\mw_l^\T \mxt}&\le 2\sqrtnd \axkh{\sqrtd + 10 \sqrtln}. 
	\end{align}
\end{subequations}
The condition \eqref{def:ric:2} reflects the approximate independence between \(\mxt\) and the noise matrix \(\mw = (\mw_1, \cdots, \mw_n)^\T\). 
As demonstrated in Lemma \ref{le:error}, the error contraction of the iterative sequence \(\dkh{\mxt}_{t=0}\) can be deduced from the RIC. 
We subsequently use induction to show that \(\dkh{\mxt}_{t=0}\) remains contained within this region throughout. 

\begin{lemma} \label{le:error}
	Suppose the noise level 
	\begin{equation}\label{condition:noise}
		\sigma \le c_0 \frac{\sqrt{nd^{-1}}}{\sqrtd + 10 \sqrtln}
	\end{equation}
	for some constant \(c_0 > 0\), and \(\etf \le \bar{e}_F \le O(1)\).  There exists an event that does not depend on $t$ and has probability at least $1 - \exp (-nd/2) - O(n^{-40})$, such that when it happens and 
	\begin{equation} \label{condition:contract}
		\ucdf{\mxt, \mz}\le \ep \sqrtn, \ \max_{1 \le i \le n}\ucnormf{\mxit - \mqt}\le 3 \ep, \  \max_{1 \le l \le n}\ucnormf{\mw_l^\T \mxt}\le 2\sqrtnd \axkh{\sqrtd + 10 \sqrtln}
	\end{equation}
	hold for \(\mqt = \op (\mz^\T \mxt)\) and some some sufficiently small constant $\ep>0$, one has
	\begin{equation}\label{eq:conclusion:sigma}
		\max_{1\le i\le n} \ucnorms{\midd - (\mf_{i}^{t})^\T \mf_{i}^{t}} < 1 ,  
	\end{equation}
	and
	\begin{equation}\label{eq:conclusion:error}
		\ucdf{\mxtt, \mz} \le \frac 12 \ucdf{\mxt, \mz} + 28 \bar{e}_F  + 4 c_0 \sqrtn .  
	\end{equation}
	Here, \(\mfit = \mxit - \mu \optit (\mgit)\) denotes the term prior to the retraction step in Algorithm \ref{alg:2}. 
\end{lemma}
\begin{proof}
	See Section \ref{pf:error}.
\end{proof}

From Lemma \ref{le:error}, one sees that  if \eqref{condition:contract} and \(\etf \le \bar{e}_F \le O(1)\) hold for the previous $t$ iterations, then
\begin{align}
	\ucdf{\mxtt, \mz} & \le \frac 12 \ucdf{\mxt, \mz} + 28 \bar{e}_F + 4 c_0 \sqrtn \notag \\ 
	& \le \frac{1}{2^{t+1}} \ucdf{\mx^0, \mz} + 56 \bar{e}_F + 8 c_0 \sqrtn, \label{eq:conclusion:contraction} 
\end{align}
which implies the conclusion of Theorem \ref{th:main}.

It is noteworthy that \eqref{eq:conclusion:sigma} is employed to guarantee the condition for Lemma \ref{le:ns:bound}. 
It remains to demonstrate the validity of condition \eqref{condition:contract} for all $0 \le t \le t_0 \le n^{10}$. We proceed by induction: specifically, one shows that if condition \eqref{condition:contract} is satisfied for the iteration $t$, it holds for the subsequent iteration ${t+1}$ with high probability. 
Moreover, assume the initial case $t = 0$ satisfies \(e^0_F \le \bar{e}_F \le O(1)\) and \eqref{condition:contract}, then one completes the proof of Theorem \ref{th:main}. 

For the first part of the condition \eqref{condition:contract},  it follows from \eqref{eq:conclusion:contraction} that if $\ucdf{\mx^0, \mz} \le \ep \sqrtn$ for some constant $\ep > 0$, one has 
\begin{equation}\label{cond:result:1}
	\dist{\mxtt, \mz}  \le \frac{1}{2^{t+1}} \ucdf{\mx^0, \mz} + 56 \bar{e}_F + 8 c_0 \sqrtn \le \ep \sqrtn, 
\end{equation}
provided \(\ebar_F \le O(1)\) and $c_0 < \ep/32$. 
However, verifying the second and third parts of \eqref{condition:contract} is non-trivial due to the statistical dependence between $\mxt$ and the noise $\mw$. To decouple this dependence, we employ a leave-one-out analysis.

\subsection{Leave-one-out sequences for auxiliary purposes}
This section establishes that the second part of condition \eqref{condition:contract} holds with high probability for all $0\le t\le t_0\le n^{10}$. We decouple the statistical dependence between \(\{\mxt\}\) and the noise $\mw$ by introducing an auxiliary sequence $\{\mxtl\}$ ($1 \le l \le n$) that removing the $l$-th row and $l$-th column of \(\mw\).  
Specifically,  for each $1\le l \le n$, we employ the auxiliary noise matrix 
\begin{align*} 
	\mwll_{ij} =   \begin{cases}
		\mw_{ij}, & \mbox{if } i \ne l \mbox{ and } j \ne l, \\
		\mzero,   & \mbox{if } i = l \mbox{ or } j = l, 
	\end{cases}
\end{align*}
the auxiliary observation \(\ma^{(l)} = \mz \mz^\T + \sigma \mwll\), and the following leave-one-out loss function 
\begin{equation*}
	F^{(l)}(\mx) = \sum_{i \ne j}^{n} \frac 12 \ucnormf{\mx_i \mx_j^\T - \ma^{(l)}_{ij}}^2  .
\end{equation*}
Additionally, the sequence \(\{\mxtl\}\) is constructed by running Algorithm \ref{alg:leave one out} with respect to $F^{(l)} (\mx)$. 
It should be noted that the auxiliary sequence is calculated using exact computation to avoid more complex discussions of error estimation. 
Moreover, the auxiliary sequences \(\{\mxtl\}\) and Algorithm \ref{alg:leave one out} are for theoretical analysis only and are excluded from actual computations. 

\begin{algorithm}
	\caption{The $l$-th iterative sequence for group synchronization}\label{alg:leave one out}
	\begin{algorithmic}[0]
		\State \textbf{Input:} The auxiliary matrix $\ma^{(l)}$, the maximum number of iterations $T$, the step size $\mu$.
		\State \textbf{Spectral initialization: } 
		Let \(\my^{(l)} \in \R^{nd \times d}\) be the top \(d\) eigenvectors of \(\ma^{(l)} = \mz^\T \mz + \sigma \mwll\) with \((\my^{(l)})^\T \my^{(l)} = \midd\). 
		Set $\mx^{0,(l)} = \osgn(\my^{(l)})$. 
		\State \textbf{Riemannian gradient updates:} 
		\For{$t = 0, 1, 2, \cdots, T-1$} 
    		\For{$i = 1 \text{ to } n$} 
        		\begin{align*}
					\mgitl &= \sum\nolimits_{j \ne i}^{n} (\mxitl - \ma_{ij} \mx_{j}^{t,(l)}), \\
					\mfitl &= \mxit - \mu \optitl (\mgitl), \\
					\mxittl &= \op (\mfitl). 
				\end{align*}
    		\EndFor
		\EndFor
	\end{algorithmic}
\end{algorithm}

Next, we will demonstrate the incoherence condition  \eqref{condition:contract}  by using the following inductive hypotheses:
\begin{subequations}\label{induct:all}
	\begin{align}
		\label{induct:sigma}\max_{1\le i\le n} \ucnorms{\midd - (\mf_{i}^{t})^\T \mf_{i}^{t}} &< 1\\ 
		\label{induct:xt}\df{\mxt, \mz} &\le \ep \sqrtn,\\
		\label{induct:xti}\max_{1\le i\le n}\normf{\mxit - \mqt} &\le 3 \ep,\\
		\label{induct:xtl}\max_{1\le l\le n}\ucdf{\mxt, \mxtl} &\le 2 \ep,\\
		\label{induct:inh}\max_{1\le l\le n}\ucnormf{\mwl^\T \mxt} &\le 2 \sqrtnd \xkh{\sqrtd + 10 \sqrtln}.
	\end{align}
\end{subequations}
Here, \(\mqt = \op(\mz^\T \mxt)\). 
We next  demonstrate that when \eqref{induct:all} holds true up to  the $t$-th iteration,  the inductive hypotheses \eqref{induct:all} still holds for the $(t+1)$-th iteration.  It is easy to see that \eqref{induct:sigma} and \eqref{induct:xt} are direct consequences of \eqref{eq:conclusion:sigma} and \eqref{eq:conclusion:contraction}. Subsequently, \eqref{induct:xti} and \eqref{induct:xtl} are justified by the following lemma.

\begin{lemma} \label{le:induct:xti & xtl}
	Suppose the noise level $\sigma$ satisfy \eqref{condition:noise}, and the step size $\mu = 1/n$. Let $\mathcal{E}_t$ be the event where \eqref{induct:sigma}-\eqref{induct:inh} hold for the $t$-th iteration. Then there exist an event $\mathcal{E}_{t+1} \subseteq \mathcal{E}_t$ obeying $\PP\xkh{\mathcal{E}_t \cap \mathcal{E}_{t+1}^C}\le \exp (-nd/2) + O(n^{-40})$, one has 
	\begin{equation}\label{eq:xti}
		\max_{1\le i\le n}\normf{\mxitt - \mqtt} \le 3 \ep, 
	\end{equation}
	and
	\begin{equation}\label{eq:xtl}
		\max_{1\le l\le n}\ucdf{\mxtt, \mxttl} \le 2 \ep.
	\end{equation}
\end{lemma}
\begin{proof}
	See Section \ref{pf:induct:xti & xtl}.
\end{proof}


Given Lemma \ref{le:induct:xti & xtl}, the inductive hypothesis \eqref{induct:inh} can be readily verified. Specifically, it follows from Lemma \ref{le:inh:bound} that \(\max_{1\le l\le n} \ucnormf{\mw_l^\T \mxttl} \leq \sqrtnd \xkh{\sqrtd + 10 \sqrt{\log n}}\) with probability exceeding \(1 - n^{-45}\) 
due to the independence between \(\mxttl\) and \(\mwl\). 
Then with probability at least \(1 - \exp(-nd/2) - O(n^{-40})\), one can obtain 
\begin{align*}
	\max_{1\le l\le n} \ucnormf{\mwl^\T \mxtt} \stackrel{\mbox{(i)}}{\le}& \max_{1\le l\le n} \ucnormf{\mwl^\T \mxttl} + \max_{1\le l\le n} \ucnormf{\mwl^\T \axkh{\mxtt - \mxttl \mqtt}} \\
	\stackrel{\mbox{(ii)}}{\le}& \max_{1\le l\le n} \ucnormf{\mwl^\T \mxttl} + \ucnorms{\mwl} \cdot \max_{1\le l\le n} \ucdf{\mxtt, \mxttl} \\
	\stackrel{\mbox{(iii)}}{\le}& \sqrtnd \xkh{\sqrtd + 10 \sqrt{\log n}} + 3 \sqrtnd \cdot 2 \ep \\
	\stackrel{\mbox{(iv)}}{\le}& 2 \sqrtnd \xkh{\sqrtd + 10 \sqrt{\log n}}. 
\end{align*}
Here, (i) comes from the triangle inequality. And (ii) utilizes Cauchy-Schwartz and \(\mqtt = \op \xkh{(\mxttl)^\T \mxtt}\). Besides, (iii) uses Lemma \ref{le:matrix:bound} and \ref{le:inh:bound}, and (iv) holds true as long as \(\ep < 1/6\). 

\subsection{Spectral initialization}
As is established in Algorithm \ref{alg:2} and Algorithm \ref{alg:leave one out}, one has the initialization \(\mx^0 = \opn (\my) \in \R^{nd \times d}\) and \(\mx^{0,(l)} = \opn (\my^{0,(l)}) \in \R^{nd \times d}\). The validity of this stage is guaranteed by Lemma \ref{le:initial} below.

\begin{lemma} \label{le:initial}
	For any $0<\delta\le 1$, with probability at least $1 - m\exp (-c_0 n) - O(m^{-10})$, we have
	\begin{align*}
		\ucdf{\mx^0,\mz} \le \delta \sqrtn, \\
		\max_{1\le i\le n}\normf{\mx_i^0 - \mq^0} \le 3 \delta, \\
		\maxm{1\le l\le m} \ucdf{\mx^0, \mx^{0,\lkh}} \le \delta, 
	\end{align*}
provided $\sigma \le c_0 \delta \sqrt{nd^{-1}}/(\sqrtd + \sqrtln)$ for some sufficiently small constant $c_0>0$.
\end{lemma}

Note that \(\mz = (\mz_1; \cdots; \mz_n) \in \R^{nd \times d}\) with \(\mz_i \in \odgroup\). Then the observation can be viewed as \(\ma = \mz \mz^\T - \mi_{nd} + \sigma \mw\). 
Treating $\sigma \mw$ as a perturbation, the analysis for initialization centers on the dominant term $\mz \mz^\T - \mi_{nd}$, which shares the same eigenvectors as $\mz \mz^\T$. 
Subsequent analysis builds upon matrix perturbation theory and the Davis-Kahan theorem \cite{davis1970, lingshuyang2022improved}, leveraging the independence facilitated by the leave-one-out argument. 
Since the proof of Lemma \ref{le:initial} is similar to the spectral initialization phase in \cite{lingshuyang2022improved}, we omit it here.

\section{Discussion}\label{sec:discussion}
We address the nonconvex challenge of orthogonal group synchronization by proposing a hardware-efficient, Newton-Schulz based Riemannian gradient descent framework. Our method replaces computationally expensive SVD retractions with iterative updates that are highly amenable to parallelization on GPUs and TPUs. Theoretically, we employ a leave-one-out analysis to prove linear convergence to the ground truth under near-optimal noise regimes, despite the use of inexact projections. Empirically, our algorithm yields nearly 2$\times$ acceleration in both synthetic benchmarks and real-world 3D registration, effectively bridging the gap between theoretical optimality and hardware-level efficiency. 

Several intriguing issues require investigation in future research. 
First, generalizing the Newton-Schulz based inexact retraction framework to other matrix manifolds, such as the Stiefel manifold or the special Euclidean group \(\sedgroup\), is of substantial practical importance. Such extensions would directly enhance the efficiency of large-scale point cloud registration and pose-graph optimization in simultaneous localization and mapping (SLAM) problem, where computational bottlenecks remain a persistent challenge. 
Second, our highly parallelizable retraction schemes could be adapted to accommodate robust loss functions, such as the \(\ell_1\)-norm or truncated least squares, to effectively mitigate the influence of outliers and heavy-tailed noise.

\section{Technical lemmas}
\begin{lemma}\cite[Lemma 4.6]{lingshuyang2022improved}\label{le:op:stable}
	For two invertible matrices \(\mx, \my \in \R^{d \times d}\), it holds 
	\[||| \op (\mx) - \op (\my) ||| \le \frac{2 ||| \mx - \my |||}{\sigma_{\min} (\mx) + \sigma_{\min} (\mx)}, \]
	where either \(||| \cdot ||| = \ucnorms{\cdot}\) or \(||| \cdot ||| = \ucnormf{\cdot}\). 
\end{lemma}

As indicated by Lemma \ref{le:op:stable}, the matrix sign function $\op$ is stable yet sensitive to singular values.

\begin{lemma}\cite[Appendix 1]{bandeira2017tightness}\label{le:matrix:bound}
	Let \(\mw\in \R^{nd \times nd}\) be a symmetric Gaussian random matrix, then 
	\[\ucnorms{\mw} \leq 3\sqrt{nd}\]
	with probability at least \(1 - \exp( - nd/2)\).
\end{lemma}

\begin{lemma}\cite[Lemma 4.11]{lingshuyang2022improved}\label{le:inh:bound}
	Let \(\mw = \xkh{\mw_1, \mw_2, \cdots, \mw_n}\in \R^{nd \times nd}\) be a random matrix defined in \eqref{def:measurements}. Then for any fixed \(\mx \in \odn\), it holds
	\[\max_{1\le l\le n} \ucnormf{\mw_l^\T \mx} \leq \sqrtnd \xkh{\sqrtd + \gamma \sqrt{\log n}}\]
	with probability at least \(1 - n^{-\gamma^2/2 + 1}\). Specially, take \(\gamma = 10\) to obtain \(\max_{1\le l\le n} \ucnormf{\mw_l^\T \mx} \leq \sqrtnd \xkh{\sqrtd + 10 \sqrt{\log n}}\) with probability exceeding \(1 - n^{-45}\). 
\end{lemma}

\begin{lemma}\label{le:sigma}
	For any \(\mx, \my \in \odn\) satisfied \(\ucdf{\mx, \my} \le \ep \sqrt{n}\), it holds 
	\[\xkh{1 - \frac{\ep^2}{2}}n \le \sigma_{\min} (\my^\T \mx) \le \sigma_{\max} (\my^\T \mx) \le n. \]
\end{lemma}
\begin{proof}
	Noting that \(\ucnorms{\my^\T \mx} \le n\) by Cauchy-Schwartz, it holds \(0 \le \sigma_i (\my^\T \mx) \le n\) for \(i \in [n]\). 
	Recall that the distance is defined as \(\ucdf{\mx, \my} := \min_{\mq \in \odgroup} \ucnormf{\mx - \my\mq}\). Let \(\mq^* = \op(\my^\T \mx)\) be the optimal orthogonal matrix that achieves this minimum. By expanding this norm, we obtain 
	\begin{align*}
		0 \le \ucnormf{\mx - \my\mq^*}^2 &= \tr\xkh{\mx^\T \mx} + \tr\xkh{\my^\T \my} - 2\tr\xkh{\mq^{*\T}\my^\T\mx}\\
		&= 2nd - 2\ucnorms{\my^\T \mx}_*. 
	\end{align*}
	As a result, the fact \(\ucdf{\mx, \my} \le \ep \sqrt{n}\) indicates 
	\[n - \sigma_{\min}(\my^\T \mx) \le \sum_{i=1}^d \xkh{n - \sigma_i(\my^\T \mx)} \le \frac 12 \ep^2 n, \]
	which gives \(\sigma_{\min}(\my^\T \mx) \ge \axkh{1 - \frac{\ep^2}{2}}n\).
\end{proof}

\begin{lemma}\cite{lingshuyang2022improved} 
	Denote that \(\mathcal{N}_{\ep} := \dkh{\ms \in \R^{nd \times d}: \df{\ms, \mz} \le \ep \sqrtn}\) for the ground truth \(\mz \in \odn\). 
	Set \(\mx, \my \in \mathcal{N}_{\ep}\) with \(\mx, \my \in \odn\), \(\mq = \op(\my^\T \mx)\), and \(\ep \le 1/\sqrttwo\). Then one has 
	\[\ucnormf{\mz^\T \xkh{\mx - \my \mq}} \le 2 \ep \sqrtn \cdot \df{\mx, \my}. \]
\end{lemma}

\begin{lemma}\label{le:sum:normf}
	Suppose \(\mx, \my \in \R^{nd\times d}\), \(\mxu = \opn(\mx), \myu = \opn(\my)\) satisfy 
	\(\mx, \my , \mxu, \myu \in \mathcal{N}_{2\ep}\)
	and \(\ucnormf{\mx - \mxu}, \ucnormf{\my - \myu} \le e_F \le \sqrtn\) with \(\ep \le 1/4\sqrttwo\),
	then \[\ucnormf{\mz^\T \xkh{\mx - \my \mq}} \le 4\ep\sqrtn \cdot \ucdf{\mx, \my} + 11 \sqrtn e_F\]
	where \(\mq = \op(\my^\T \mx)\). 
\end{lemma}
\begin{proof}
	Define \(\mqu = \op(\myu^\T \mxu)\). The triangle inequality and the Cauchy-Schwartz inequality yield 
	\begin{align*}
		\ucnormf{\mz^\T \xkh{\mx - \my \mq}} 
		&\le \ucnormf{\mz^\T (\mxu - \myu \mqu)} + \ucnormf{\mz^\T (\mx - \mxu)} + \ucnormf{\mz^\T (\myu - \my)\mq} + \ucnormf{\mz^\T \myu (\mqu - \mq)}\\
		&\le 4\ep\sqrtn \cdot \ucdf{\mxu, \myu} + 2 \sqrtn e_F + n \ucnormf{\mqu - \mq}\\
		&\le 4\ep\sqrtn \cdot \ucdf{\mx, \my} + 4\ep\sqrtn \xkh{\ucdf{\mx, \mxu} + \ucdf{\my, \myu}} + 2 \sqrtn e_F + n \ucnormf{\mqu - \mq}
	\end{align*}
	where the second inequality comes from \cite[Proof of Lemma 4.5]{lingshuyang2022improved}. Then one can see that 
	\[\ucdf{\mx, \mxu} + \ucdf{\my, \myu} \le \ucnormf{\mx - \mxu} + \ucnormf{\my - \myu} \le 2 e_F.\]
	To bound the remaining item, the triangle inequality and Cauchy-Schwartz give that 
	\begin{align*}
		n \ucnormf{\mqu - \mq}
		& = n \ucnormf{\op(\my^\T \mx) - \op(\myu^\T \mxu)}
		\stackrel{\mbox{(i)}}{\le} \frac{2n}{\sigma_{\min} (\myu^\T \mxu)} \ucnormf{\my^\T \mx - \myu^\T \mxu}\\
		&\hspace{-0.25em}\stackrel{\mbox{(ii)}}{\le} \hspace{-0.25em} 3 \ucnormf{\my^\T \mx - \myu^\T \mxu}\\
		&\le 3 \xkh{\ucnorms{\myu} + \ucnormf{\myu-\my}} \ucnormf{\mx-\mxu} + 3\ucnorms{\mxu} \ucnormf{\my-\myu}\\
		&\le 3 (\sqrtn + e_F) \cdot e_F + 3 \sqrtn e_F \le 9 \sqrtn e_F. 
	\end{align*}
	Here, (i) holds since Lemma \ref{le:op:stable}, and (ii) utilizes \(\sigma_{\min} (\myu^\T \mxu) \ge (1 - 8 \ep^2) n \ge \frac 34 n\) due to Lemma \ref{le:sigma}. Therefore, it holds 
	\[
	\ucnormf{\mz^\T \xkh{\mx-\my \mq}} \le 4\ep\sqrtn \cdot \ucdf{\mx, \my} + 12 \sqrtn e_F. 
	\]
\end{proof}

\section{Appendix B: Proof of Lemma \ref{le:error}}\label{pf:error}
As stated in Algorithm \ref{alg:2} and Section \ref{sec:results}, one has 
\begin{align*}
	\mxit = \os(\mfit), \quad \optit (\mgit) = \frac 12 \xkh{\mgit - \mxit (\mgit)^\T \mxit}, \quad \mxitu = \op(\mfit). 
\end{align*}
We abuse the notation and denote 
\[\mxt = \osn(\mft), \quad \optitu (\mgit) = \frac 12 \xkh{\mgit - \mxitu (\mgit)^\T \mxitu}, \quad \mxtu = \opn(\mft).\]

We first look at the result \eqref{eq:conclusion:sigma}. For each \(i \in [n]\), it follows from the triangle inequality that 
\begin{align*}
	\ucnormf{\mfit - \mqt} &= \ucnormf{\mxit - \mu \optit (\mgit) - \mqt} \\
	&\le \ucnormf{\mxit - \mqt} + \mu \ucnormf{\optitu (\mgit)} + \mu \ucnormf{\optitu (\mgit) - \optit (\mgit)} \\
	&\le 3 \ep + \mu \ucnormf{\mgit} + \mu \ucnormf{\optitu (\mgit) - \optit (\mgit)}, 
\end{align*}
where the last line comes from \(\ucnormf{\optitu (\cdot)} \le \ucnormf{\cdot}\) and the condition \eqref{condition:contract}. 
For the term \(\mu \ucnormf{\mgit}\), recall that \(\mgit = \sum_{j\ne i}^{n} (\mxit - \ma_{ij}\mxit) = n(\mxit - \mqt) - \mz^\T (\mxt - \mz \mqt) - \sigma \mw_i^\T \mxt\) since \(\mw_{kk} = \mzero\). As a result, one sees 
\begin{align}
	\max_{1\le i \le n} \mu \ucnormf{\mgit} \le& \ \max_{1\le i \le n}\ucnormf{\mxit - \mqt} + \frac{1}{n} \ucnormf{\mz^\T (\mxt - \mz \mqt)} + \frac{\sigma}{n} \max_{1\le i \le n}\ucnormf{\mw_i^\T \mxt} \notag\\
	\le& \ \max_{1\le i \le n}\ucnormf{\mxit - \mqt} + \frac 1n \xkh{4\ep \sqrtn \cdot \df{\mxt, \mz} + 12 \sqrtn \etf} + 2c_0 \notag\\ 
	\le& \ 3\ep + \frac 1n \xkh{4\ep \sqrtn \cdot \ep + 12 \sqrtn \etf} + 2c_0 
	\le \frac 18, \label{eq:1n git}
\end{align}
because of \(\mu = 1/n\), Lemma \ref{le:sum:normf}, the condition \eqref{condition:contract}, and
\[\ep < \frac{1}{32},\quad \ucnormf{\mxt - \mxtu} = \etf < \ebar_F < \frac{1}{256},\quad c_0 < \frac{1}{128}. \]
It remains to derive the last item \(\mu \ucnormf{\optitu (\mgit) - \optit (\mgit)}\). 
The triangle inequality and Cauchy-Schwarz give that  
\begin{align}
	\mu \|\optitu (&\mgit) - \optit (\mgit)\|_F = \frac{1}{2n} \ucnormf{\mxitu (\mgit)^\T \mxitu - \mxit (\mgit)^\T \mxit} \notag\\
	&\le \frac{1}{2n} \ucnormf{\mgit} \ucnorms{\mxitu - \mxit} \xkh{2\ucnorms{\mxitu} + \ucnorms{\mxitu - \mxit}} \le \frac{1}{4} \ucnorms{\mxitu - \mxit} \label{eq:1n delta ps pt}
\end{align}
for each \(i \in [n]\) since \(\mxitu (\mxitu)^\T = \midd\), \(\ucnorms{\mxitu - \mxit} \le \etf \le 1\), and \eqref{eq:1n git}. 
This allows us to derive 
\begin{align}\label{eq:sigma fit}
	\max_{1 \le i \le n} \ucnorms{\mfit - \mqt} \le \max_{1 \le i \le n} \ucnormf{\mfit - \mqt} \le 3 \ep + \max_{1 \le i \le n} \mu \ucnormf{\mgit} + \frac 14 \etf  \le \frac 14, 
\end{align}
and \(\max_{1\le i\le n} \ucnorms{\midd - (\mf_{i}^{t})^\T \mf_{i}^{t}} < 1\) because of \(\mu = 1/n\) and the Weyl's inequality. 

Next, we move on to the result \eqref{eq:conclusion:error}. 
By definition of $\df{\cdot,\cdot}$, it is easy to see that
\begin{equation*}
	\begin{aligned}
		\df{\mxtt,\mz} = \min_{\mq\in \odgroup} \ucnormf{\mxtt-\mz \mq} \le \ucnormf{\mxtt-\mz \mqt}
	\end{aligned}
\end{equation*}
where \(\mqt = \op (\mz^\T \mxt)\). Thus, the iterations \(\mxtt = \osn(\mft) = \osn(\mxt - \mu \optt(\mgt))\) and \(\mxttu = \opn(\mft)\) reveals that  
\begin{equation}\label{eq:df xtt z}
		\df{\mxtt,\mz} \le \ucnormf{\mxtt - \mxttu} + \ucnormf{\mxttu - \mz \mqt}
		\le \ettf + 2 \ucnormf{\mft - \mz \mqt}
\end{equation}
where the first inequality uses the triangle inequality, and the second inequality utilizes the facts that 
\[\mxttu - \mz \mqt = \xkh{\op(\mf_1^t) - \mqt; \cdots; \op(\mf_n^t) - \mqt}\]
and 
\[\ucnormf{\op(\mfit) - \mqt} = \ucnormf{\op(\mfit) - \op(\mqt)} \le \frac{2}{\sigma_{\min}(\mqt)} \ucnormf{\mfit - \mqt} = 2 \ucnormf{\mfit - \mqt}\]
by Lemma \ref{le:op:stable}. Additionally, one has 
\begin{equation*}
	\begin{aligned}
		\ucnormf{\mft - \mz \mqt} 
		\stackrel{\mbox{(i)}}{\le}& \ucnormf{(\oi - \opttu)(\mxt - \mxtu)} + \ucnormf{(\oi - \opttu)(\mxtu - \mz \mqt)}\\
		& + \ucnormf{\opttu(\mxt - \mu \mz \mqt) - \opttu(\mgt)} + \mu \ucnormf{\opttu(\mgt) - \optt(\mgt)}\\
		\stackrel{\mbox{(ii)}}{\le}& \etf + \underbrace{\ucnormf{(\oi - \opttu)(\mxtu - \mz \mqt)}}_{:=I_1} + \underbrace{\ucnormf{\mxt - \mz \mqt - \mu \mgt}}_{:=I_2}\\
		& + \underbrace{\mu \ucnormf{\opttu(\mgt) - \optt(\mgt)}}_{:=I_3}.
	\end{aligned}
\end{equation*}
where we abuse the notation and denote \(\optt(\mgt) := ({\mathcal P}_{T_{\mx_1^t}}(\mgt_1); \cdots; {\mathcal P}_{T_{\mx_n^t}}(\mgt_n) )\), \(\opttu(\mgt) := ({\mathcal P}_{T_{\mxtu_1}}(\mgt_1); \cdots; {\mathcal P}_{T_{\mxtu_n}}(\mgt_n) )\). 
Here, (i) arises from the triangle inequality. And (ii) is a consequence of 
\(\ucnormf{(\oi - \opttu)(\mxt - \mxtu)} \le \ucnormf{\mxt - \mxtu}\le \etf\) and \(\ucnormf{\opttu (\cdot)}\le \ucnormf{\cdot}\). 
In the sequel, we control the above three terms \(I_1\), \(I_2\) and \(I_3\) separately. 

\begin{itemize}
\item For the first term \(I_1\), we have 
\begin{align*}
	\ucnormf{(\oi - \opttu)(\mxtu - \mz \mqt)} &=  \sqrt{\sum_{i=1}^{n} \ucnormf{(\oi - \optitu)(\mxtu_i - \mqt)}^2} \\
	&\le \sqrt{\sum_{i=1}^{n} \ucnormf{\mxtu_i - \mqt}^2 \cdot \frac 14 \cdot \frac{1}{4^2}} 
	=  \frac{1}{8} \ucnormf{\mxtu - \mz \mqt}, 
\end{align*}
where the inequality arises from Cauchy-Schwartz and 
\begin{eqnarray}\label{eq:i - pt}
	(\oi - \optitu) (\mxitu - \mq) &=  & \frac{\mxitu \xkh{(\mxitu)^\T\axkh{\mxitu - \mq} + \axkh{\mxitu - \mq}^\T \mxitu}}{2} \\
	&= &  \frac{\mxitu (\mxitu - \mq)^\T (\mxitu - \mq)}{2}  \notag
\end{eqnarray}
for any \(\mq \in \odgroup\) and \[\ucnormf{\mxitu - \mqt} \le \ucnormf{\mxitu - \mxit} + \ucnormf{\mxit - \mqt} \le \etf + 3 \ep \le \frac 14\]
due to \(\etf \le \frac{1}{8}\) and \(\ucnormf{\mxit - \mqt} \le 3 \ep \le \frac{1}{8}\). 
As a result, 
\begin{align*}
	I_1 \le \frac{1}{8} \ucnormf{\mxtu - \mz\mqt} 
	\le \frac{1}{8} \xkh{\ucnormf{\mxtu - \mxt} + \ucnormf{\mxt - \mz\mqt}} 
	\le \frac{1}{8} \etf + \frac{1}{8}\df{\mxt, \mz}. 
\end{align*}

\item For the second term \(I_2\), recall that 
\(\mgt = n \mxt - \mz \mz^\T \mxt - \sigma \mw \mxt\) and \(\mu = 1/n\) to obtain 
\begin{align*}
	\mxt - \mz\mqt - \mu \mgt = \frac 1n \mz (\mz^\T \mxt - \mz^\T \mz \mqt). 
\end{align*}
Then it follows from the triangle inequality and the Cauchy-Schwarz inequality that 
\begin{align*}
	I_2 =&\ \ucnormf{\mxt - \mz \mqt - \mu \mgt} \\
	\le&\ \frac{\sqrtn}{n} \ucnormf{\mz^\T (\mxt - \mz \mqt)} + \frac{\sigma \sqrtn}{n} \max_{1\le k \le n}\ucnormf{\mw_k^\T \mxt} \\
	\stackrel{\mbox{(i)}}{\le}&\ \frac{\sqrtn}{n} \xkh{4\ep \sqrtn \cdot \df{\mxt, \mz} + 12 \sqrtn \etf} + \frac{\sigma \sqrtn}{n} \max_{1\le k \le n}\ucnormf{\mw_k^\T \mxt} \\
	\stackrel{\mbox{(ii)}}{\le}&\ \frac{1}{8} \df{\mxt, \mz} + 12 \etf + 2 c_0 \sqrtn.
\end{align*}
Here, (i) utilizes Lemma \ref{le:sum:normf} with the conditions \(\df{\mxt, \mz} \le \ep \sqrtn\), \(\ucnormf{\mxtu - \mxt} \le \etf \le \ep\), and \(\ucdf{\mxtu, \mz}\le \ucdf{\mxt, \mz} + \ucnormf{\mxtu - \mxt} \le 2 \ep \sqrtn\). Besides, (ii) holds due to \(\ep \le \frac{1}{32}\) and condition \eqref{condition:contract} that 
\[\max_{1\le k \le n}\ucnormf{\mw_k^\T \mxt} \le 2\sqrtnd \axkh{\sqrtd + 10 \sqrtln}. \]

\item It remains to control the third term $I_3$. Following the discussion in \eqref{eq:1n git} and \eqref{eq:1n delta ps pt}, one sees 
\[\max_{1\le i \le n} \mu \|\optitu (\mgit) - \optit (\mgit)\|_F \le \frac{1}{4} \ucnorms{\mxitu - \mxit}. \]
This allows us to derive 
\begin{align}
	I_3 &= \mu \ucnormf{\opttu (\mgt) - \optt (\mgt)} = \sqrt{\sum_{i=1}^{n} \mu^2 \ucnormf{\optitu (\mgit) - \optit (\mgit)}^2} \notag\\
	&\le \frac 14 \ucnormf{\mxtu - \mxt} \le \frac 14 \etf. \label{eq:i3}
\end{align}

\end{itemize}

Finally, putting the above estimates on \(I_1, I_2\) and \(I_3\) together, we conclude that 
\begin{equation}\label{eq:delta ft zqt}
	\ucnormf{\mft - \mz \mqt} \le \etf + I_1 + I_2 + I_3 \le 14\; \df{\mxt,\mz} + \frac{27}{2} \etf + 2 c_0 \sqrtn. 
\end{equation}
Substituting this upper bound into \eqref{eq:df xtt z} yields  
\begin{align*}
	\df{\mxtt,\mz} \le \ettf + 2 \ucnormf{\mft - \mz \mqt} \le \frac 12 \df{\mxt,\mz} + 28 \ebar_F + 4 c_0 \sqrtn, 
\end{align*}
which completes the proof.

\section{Appendix C: Proof of Lemma \ref{le:induct:xti & xtl}}\label{pf:induct:xti & xtl}
\subsection{Proof of result (\ref{eq:xti})}

Regarding the triangle inequality, one can use the decomposition
\begin{align*}
	\ucnormf{\mxitt - \mqtt} &\le \ucnormf{\mxitt - \mxittu} + \ucnormf{\mxittu - \mqt} + \ucnormf{\mqt - \mqtt} \\
	&= \ucnormf{\mxitt - \mxittu} + \ucnormf{\op(\mfit) - \op(\mqt)} + \ucnormf{\mqt - \mqtt} \\
	&\le \ettf + \underbrace{\frac 87 \ucnormf{\mfit - \mqt}}_{:=I_4} + \underbrace{\ucnormf{\mqt - \mqtt}}_{:=I_5}. 
\end{align*}
Here, Lemma \ref{le:op:stable} gives 
\[\ucnormf{\op(\mfit) - \op(\mqt)} \le \frac{2}{\sigma_{\min}(\mfit) + \sigma_{\min}(\mqt)} \ucnormf{\mfit - \mqt} \le \frac 87 \ucnormf{\mfit - \mqt}, \] 
since \(\sigma_{\min}(\mqt) = 1\) and \(\sigma_{\min}(\mfit) \ge 3/4\) by \eqref{eq:sigma fit} and the Weyl's inequality. 
In the sequel, we shall bound the two terms \(I_4\) and \(I_5\) separately.

\begin{itemize}
	\item With regard to \(I_4\), one can obtain 
\begin{align*}
	\ucnormf{\mfit - \mqt} 
	\le& \ \ucnormf{(\oi - \optitu) (\mxit - \mxitu)} + \ucnormf{(\oi - \optitu) (\mxitu - \mqt)} \\
	&+ \ucnormf{\optitu (\mxit - \mqt) - \mu \optitu(\mgit)} + \mu \ucnormf{\optitu(\mgit) - \optit(\mgit)} \\
	\le& \ \etf + \ucnormf{(\oi - \optitu) (\mxitu - \mqt)} + \ucnormf{\mxit - \mqt - \mu \mgit} \\
	&+ \mu \ucnormf{\optitu(\mgit) - \optit(\mgit)} .  
\end{align*}
Use similar estimation in Section \ref{pf:error} to derive 
\begin{align*}
	\ucnormf{(\oi - \optitu) (\mxitu - \mqt)} &\le \frac 18 \ucnormf{\mxitu - \mqt} \le \frac 18 \ucnormf{\mxit - \mqt} + \frac 18 \etf ,  \\
	\ucnormf{\mxit - \mqt - \mu \mgit} &\le  \frac{1}{n} \ucnormf{\mz^\T (\mxt - \mz \mqt)} + \frac{\sigma}{n} \max_{1\le k \le n}\ucnormf{\mw_k^\T \mxt} \\
	&\le \frac{1}{n} \xkh{4\ep \sqrtn \cdot \df{\mxt, \mz} + 12 \sqrtn \etf} + \frac{\sigma}{n} \max_{1\le k \le n}\ucnormf{\mw_k^\T \mxt} \\
	&\le \frac{1}{8 \sqrtn} \df{\mxt, \mz} + 12 \etf + 2 c_0 ,  
\end{align*}
and  
\begin{align*}
	\mu \ucnormf{\optitu(\mgit) - \optit(\mgit)} \le \frac{1}{4} \ucnorms{\mxitu - \mxit} \le \frac 14 \etf 
\end{align*}
as long as \(\ep < 1/32\). 
Putting together the above bounds, we reach 
\begin{align}
	I_4 &\le \frac{8}{7} \xkh{\frac 18 \ucnormf{\mxit - \mqt} + \frac{1}{8 \sqrtn} \df{\mxt, \mz} + 14 \etf + 2 c_0} \notag \\
	&\le \frac 17 \ucnormf{\mxit - \mqt} + \frac{1}{7 \sqrtn} \df{\mxt, \mz} + 16 \etf + 3 c_0 . \label{eq:i4}  
\end{align}

\item Next, we need to control \(I_5\). 
We claim that 
\begin{equation}\label{claim:sigma}
	\sigma_{\min} (\mz^\T \mxt) \ge \frac{3}{4} n, \quad\sigma_{\min} (\mz^\T \mxtt) \ge \frac{3}{4} n. 
\end{equation}
As a result, one can invoke Lemma \ref{le:op:stable} to get 
\begin{align*}
	\|\mqt - \mqtt\|_F =&\; \normf{\op \xkh{\mz^\T \mxt} - \op \xkh{\mz^\T \mxtt}} \\ 
	\le&\ \frac{2}{\sigma_{\min} (\mz^\T \mxt) + \sigma_{\min} (\mz^\T \mxtt) } \normf{\mz^\T (\mxt - \mxtt)} \\ 
	\stackrel{\mbox{(i)}}{\le}&\ \frac{4}{3 \sqrtn} \xkh{\ucnormf{\mxttu - \mxtu} + \ucnormf{\mxtt - \mxttu} + \ucnormf{\mxt - \mxtu}} \\ 
	\stackrel{\mbox{(ii)}}{\le}&\ \frac{4}{3 \sqrtn} \xkh{\frac 87 \normf{\mu \optt (\mgt)} + \ucnormf{\mxtt - \mxttu} + \frac{15}{7}\ucnormf{\mxt - \mxtu}} \\
	\stackrel{\mbox{(iii)}}{\le}&\ \frac{32}{21 \sqrtn} \xkh{\normf{\mxt - \mz\mqt} + \normf{\mxt - \mz\mqt - \mu \optt (\mgt)}} + 4\ebarf .  
\end{align*}
Here, the inequality (i) comes from \(\norms{\mz} = \sqrtn\) and the triangle inequality. 
The second inequality (ii) utilizes the facts that 
\begin{align*}
	\ucnormf{\mxttu - \mxtu} = \ucnormf{\osgn(\mft) - \osgn(\mxtu)} &\le \frac{2}{\min_{1 \le i \le n} \sigma_{\min} (\mfit) + 1 } \ucnormf{\mft - \mxtu} \\
	&\le \frac 87 \ucnormf{\mxt - \mxtu - \mu \optt (\mgt)}
\end{align*}
due to Lemma \ref{le:op:stable} and \(\sigma_{\min}(\mfit) \ge 3/4\) by \eqref{eq:sigma fit}. 
And (iii) arises from the triangle inequality and \(\ucnormf{\mxtt - \mxttu}, \ucnormf{\mxt - \mxtu} \le \ebarf\). 
In addition, recognize that \(\ucnormf{\mxt - \mz \mqt} = \df{\mxt, \mz}\) and 
\begin{align*}
	\normf{\mxt - \mz\mqt - \mu \optt (\mgt)} &= \normf{\mft - \mz\mqt} \le \frac 14 \df{\mxt,\mz} + 14 \etf + 2 c_0 \sqrtn   
\end{align*}
from \eqref{eq:delta ft zqt} to obtain 
\begin{align}
	I_5 &= \|\mqt - \mqtt\|_F \notag \\
	&\le \frac{32}{21 \sqrtn} \xkh{\df{\mxt, \mz} + \frac 14 \df{\mxt,\mz} + 14 \etf + 2 c_0 \sqrtn} + 4 \ebarf \notag \\ 
	&\le \frac{40}{21 \sqrtn} \df{\mxt, \mz} + 26 \etf + 4 c_0 .  \label{eq:i5} 
\end{align}
Now we are going to prove \eqref{claim:sigma}. It is easy to verify that \(\mxt, \mxtu \in \mathcal{N}_{2\ep}\), which implies \(\sigma_{\min} (\mz^\T \mxtu) \ge (1- 2 \ep^2)n\) by Lemma \ref{le:sigma}. 
Applying the Weyl's inequality and Cauchy-Schwartz yields 
\[\abs{\sigma_{\min} (\mz^\T \mxt) - \sigma_{\min} (\mz^\T \mxtu)} \le \ucnorms{\mz^\T (\mxt - \mxtu)} \le \sqrtn \etf.  \]
As a consequence,  
\[\sigma_{\min} (\mz^\T \mxt) \ge \xkh{1 - 2 \ep^2}n - \sqrtn \etf \ge \frac{3}{4} n  \]
as long as \(\ep \le 1/4\) and \(\etf \le 1/8 \). Similarly, one can show that \(\sigma_{\min} (\mz^\T \mxtt) \ge 3n/4\). 
\end{itemize}

To finish up, combining the bounds obtained in \eqref{eq:i4} and \eqref{eq:i5}, we arrive at 
\begin{align*}
	\maxm{1 \le i \le n}\ucnormf{\mxitt - \mqtt} &\le \frac 17 \maxm{1 \le i \le n}\ucnormf{\mxit - \mqtt} + \frac{43}{21\sqrtn} \df{\mxt, \mz} + \ettf + 42 \ebarf + 7 c_0 \\
	&\le \frac 17 \cdot 3 \ep + \frac{43}{21\sqrtn} \cdot \ep \sqrtn + 43 \etf + 7 c_0 \le 3 \ep 
\end{align*}
due to the incudtion hypotheses \eqref{induct:all}, and \(\etf \le \frac{\ep}{120}\), \(c_0 \le \frac{\ep}{56}\).  

\subsection{Proof of result (\ref{eq:xtl})}
As mentioned previously, one has \(\mxitt = \os(\mfit), \mxittl = \op(\mfitl)\) and 
\begin{align*}
	\mfit = \mxit - \mu \optit (\mgit) ,\quad \mfitl = \mxitl - \mu \optitl (\mgitl) 
\end{align*}
for \(i \in [n]\) from Algorithm \ref{alg:2} and Algorithm \ref{alg:leave one out}.  
Denote \(\sigma_d = \min_{1 \le i \le n} \sigma_{\min}(\mfit)\) and \(\sigma_{d}^{(l)} = \min_{1 \le i \le n}\sigma_{\min}(\mfitl)\). 
As a result, for \(l \in [n]\), by invoking the triangle inequality and Lemma \ref{le:op:stable}, we reach 
\begin{align*}
	\ucdf{\mxitt, \mxittl} &\le \ucnormf{\osn(\mft) - \opn(\mftl) \mqtl} \\
	&\le \ucnormf{\osn(\mft) - \opn(\mft)} + \ucnormf{\opn(\mft) - \opn(\mftl \mqtl)} \\
	&\le \ettf + \frac{2}{\sigma_d + \sigma_{d}^{(l)}} \ucnormf{\mft - \mftl \mqtl} ,  
\end{align*}
where \(\mqtl = \op ((\mxtl)^\T \mxt)\), and we utilize \(\opn(\mftl) \mqtl = \opn(\mftl \mqtl)\) and 
\[\ucnormf{\osn(\mft) - \opn(\mft)} = \ucnormf{\mxtt - \mxttu} \le \ettf. \] 
The remaining proof consists of two steps: (1) demonstrating that
\begin{align}\label{claim:delta f fl}
	\ucnormf{\mft - \mftl \mqtl} &\le \frac 12 \ucdf{\mxt, \mxtl} + 14 \etf + 4 c_0 ,  
\end{align}
and (2) showing that
\begin{align}\label{claim:sigma:frac}
	\frac{2}{\sigma_d + \sigma_{d}^{(l)}} &\le \frac 32 .  
\end{align}

Regarding \eqref{claim:delta f fl}, one can further decompose 
\begin{align*}
	\|\mft - &\mftl \mqtl\| \le \ucnormf{(\oi - \opttu) (\mxt - \mxtu)} + \ucnormf{(\oi - \opttu) (\mxtu - \mxtl \mqtl)} \\
	&+ \ucnormf{\opttu (\mxt - \mxtl \mqtl) - \mu \opttu(\mgt - \mgtl \mqtl)} \\
	&+ \mu \ucnormf{\opttu(\mgt) - \optt(\mgt)} + \mu \ucnormf{\opttu(\mgtl \mqtl) - \opttl(\mgtl)\mqtl} \\
	\le& \etf + \underbrace{\ucnormf{(\oi - \opttu) (\mxtu - \mxtl \mqtl)}}_{:=I_6} \\
	&+ \underbrace{\ucnormf{ \mxt - \mxtl \mqtl - \mu (\mgt - \mgtl \mqtl)}}_{:=I_7} + \underbrace{\mu \ucnormf{\opttu(\mgt) - \optt(\mgt)}}_{:=I_8} \\
	&+ \underbrace{\mu \ucnormf{\opttu(\mgtl \mqtl) - \opttl(\mgtl)\mqtl}}_{:=I_9} .  
\end{align*}
where we abuse the notation and denote \(\opttl(\mgt) := ({\mathcal P}_{T_{\mxtl_1}}(\mgt_1); \cdots; {\mathcal P}_{T_{\mxtl_n}}(\mgt_n) )\). 
In the sequel, we bound these four terms separately. 

\begin{itemize}
\item With regards to \(I_6\), we have
\begin{align*}
	I_6 &= \sqrt{4 \sum_{i=1}^{n} \ucnormf{(\oi - \optitu)(\mxtu_i - \mxitl \mqtl)}^2} 
	\le \sqrt{4 \sum_{i=1}^{n} \ucnormf{\mxtu_i - \mxitl \mqtl}^2 \cdot \frac 14 \cdot \xkh{\frac{3}{32}}^2} \\
	&=  \frac{3}{32} \ucnormf{\mxtu - \mxtl \mqtl} ,  
\end{align*}
since \eqref{eq:i - pt} and 
\begin{align*}
	\ucnormf{\mxitu - \mxitl \mqtl} &\le \ucnormf{\mxitu - \mxit} + \ucnormf{\mxit - \mxitl \mqtl} \\
	&= \ucnormf{\mxitu - \mxit} + \ucdf{\mxit, \mxitl} \le \etf + 2 \ep \le \frac{3}{32}
\end{align*}
for any \(i \in [n]\) due to the indcution hypothesis \eqref{induct:xtl} and \(\ep \le \frac{1}{32}, \etf \le \frac{1}{32}\). 
As a result, 
\begin{align*}
	I_6 \le \frac{3}{32} \ucnormf{\mxtu - \mxtl \mqtl} &\le \frac{3}{32} \xkh{\ucnormf{\mxtu - \mxt} + \ucnormf{\mxt - \mxtl\mqtl}} \\
	&\le \frac{3}{32} \etf + \frac{3}{32} \ucdf{\mxt, \mxtl} .  
\end{align*}

\item Regarding \(I_7\), recall that 
\begin{align*}
	\mgt = n \mxt - \mz \mz^\T \mxt - \sigma \mw \mxt , \mbox{ and }\ 
	\mgtl = n \mxtl - \mz \mz^\T \mxtl - \sigma \mwll \mxtl, 
\end{align*}
which indicates 
\begin{align*}
	& \mxt - \mxtl \mqtl - \mu (\mgt - \mgtl \mqtl) \\
	=\ & \frac 1n \mz \mz^\T (\mxt - \mxtl \mqtl) + \frac{\sigma}{n} \xkh{\mw \mxt - \mwll \mxtl \mqtl} .  
\end{align*}
So one can derive that 
\begin{align*}
	I_7 = \ucnormf{\mxt - \mxtl \mqtl - \mu (\mgt - \mgtl \mqtl)} \le \frac{\sqrtn}{n} \ucnormf{\mz^\T (\mxt - \mxtl \mqtl)} \\
	+ \frac{\sigma}{n} \ucnorms{\mw}\ucnormf{\mxt - \mxtl \mqtl} + \frac{\sigma}{n} \ucnormf{(\mw - \mwll)^\T \mxtl} .  
\end{align*}
It is easy to see that \(\df{\mxt, \mz} \le  \ep \sqrtn\), \(\ucnormf{\mxt - \mxtu} \le \ebarf \le \sqrtn\), \(\mxtl \in \odn\), and \(\df{\mxtl, \mz} \le 2 \ep \sqrtn\). 
Then Lemma \ref{le:sum:normf} gives  
\begin{align*}
	\frac{\sqrtn}{n} \ucnormf{\mz^\T (\mxt - \mxtl \mqtl)} &\le \frac{\sqrtn}{n} \xkh{4\ep\sqrtn \cdot \ucdf{\mxt, \mxtl} + 12 \sqrtn \etf} \\
	&\le \frac 18 \ucdf{\mxt, \mxtl} + 12 \etf, 
\end{align*}
as long as \(\ep \le 1/32\). In addition, we obtain 
\begin{align*}
	\frac{\sigma}{n} \ucnorms{\mw}\ucnormf{\mxt - \mxtl \mqtl} = \frac{\sigma}{n} \ucnorms{\mw} \cdot \ucdf{\mxt, \mxtl} \le \frac{1}{32} \ucdf{\mxt, \mxtl}, 
\end{align*}
and 
\begin{align*}
	\frac{\sigma}{n} \ucnormf{(\mw - \mwll)^\T \mxtl} &\le \frac{\sigma}{n} \xkh{\ucnormf{\mwl \mxltl} +  \ucnormf{\mwl^\T \mxtl}} \\
	&\le \frac{\sigma}{n} \xkh{\sqrtd \cdot \ucnorms{\mwl} +  \sqrtnd (\sqrtd + 10 \sqrtln)} \le 4c_0 \ep , 
\end{align*}
with probability at least \(1 - \exp(-nd/2) - O(n^{-20})\) since 
\begin{align*}
	\ucnorms{\mwl} \le \ucnorms{\mw} \le 3 \sqrtnd, \mbox{ and }\ \ucnormf{\mwl^\T \mxtl} \le \sqrtnd (\sqrtd + 10 \sqrtln), 
\end{align*}
from Lemma \ref{le:matrix:bound} and Lemma \ref{le:inh:bound}. These together yield 
\begin{align}\label{eq:i7}
	I_7 &\le \frac{5}{32} \ucdf{\mxt, \mxtl} + 12 \etf + 4c_0 \ep .  
\end{align}

\item Since \(I_8\) and \(I_3\) have the same form, we can upper bound \(I_8\) by 
\begin{align}\label{eq:i8}
	I_8 = \mu \ucnormf{\opttu(\mgt) - \optt(\mgt)} = I_3 \le \frac{1}{4} \etf 
\end{align}
in view of \eqref{eq:i3}. 

\item The next term we shall control is \[I_9 = \mu \ucnormf{\opttu(\mgtl \mqtl) - \opttl(\mgtl)\mqtl}. \] 
By the definition of \(\opttu\) and \(\opttl\), considering each block, we have
\begin{align*}
	&\ \mu \ucnormf{\optitu(\mgitl \mqtl) - \optitl(\mgitl)\mqtl} \\
	\le&\ \frac{1}{2n} \normf{\mxitu \axkh{\mgitl \mqtl}^\T \mxitu - \mxitl \axkh{\mgitl}^\T \mxitl \mqtl} \\
	\le&\ \frac{1}{2n} \xkh{\normf{\xkh{\mxitu (\mqtl)^\T - \mxitl} (\mgitl)^\T \mxitu} + \normf{\mxitl \axkh{\mgitl}^\T \xkh{\mxit - \mxitl \mqtl}}} \\
	\le&\ \frac{1}{n} \ucnormf{\mgitl} \ucnormf{\mxitu - \mxitl \mqtl} ,  
\end{align*}
for each \(i \in [n]\) due to the triangle inequality and Cauchy-Schwarz. 
It then boils down to bounding \(\frac 1n \ucnormf{\mgitl}\), towards which we decompose
\begin{align*}
	& \frac 1n \ucnormf{\mgit - \mgitl\mqtl} \\
	 \le\, &\; \frac 1n \ucnormf{\mgt - \mgtl \mqtl} \\
	\stackrel{\mbox{(i)}}{\le}&\; \ucnormf{\mxt - \mxtl \mqtl} + \ucnormf{\mxt - \mxtl \mqtl - \mu \mgt + \mu \mgtl \mqtl} \\
	=\, &\; \ucdf{\mxt, \mxtl} + I_7 \\
	\stackrel{\mbox{(ii)}}{\le}&\; \frac{37}{32} \ucdf{\mxt, \mxtl} + 12 \etf + 4c_0 \ep \stackrel{\mbox{(iii)}}{\le} \frac 18 . 
\end{align*}
for each \(i, l \in [n]\). Here, (i) arises due to the triangle inequality. (ii) comes from \eqref{eq:i7} that 
\begin{align*}
	I_7 \le \frac{5}{32} \ucdf{\mxt, \mxtl} + 12 \etf + 4c_0 \ep . 
\end{align*}
And (iii) relies on 
\[\ucdf{\mxt, \mxtl} \le 2 \ep \le \frac{1}{16}, \ \etf \le \frac{1}{400}, \mbox{ and } c_0 \le \frac{1}{8}. \]
It has been shown in \eqref{eq:1n git} that \(\max_{1\le i \le n}\frac 1n \ucnormf{\mgit} \le \frac 18\). As a result, 
\begin{align*}
	\mu \|\optitu(\mgitl \mqtl) &- \optitl(\mgitl)\mqtl\|_F \le \frac{1}{n} \ucnormf{\mgitl} \ucnormf{\mxitu - \mxitl \mqtl} \\
	&\le \xkh{\frac 1n \ucnormf{\mgit} + \frac 1n \ucnormf{\mgit - \mgitl\mqtl}} \ucnormf{\mxitu - \mxitl \mqtl} \\
	&\le \frac{1}{4} \ucnormf{\mxitu - \mxitl \mqtl} 
\end{align*}
for each \(i, l \in [n]\). This leads to an upper bound 
\begin{align*}
	I_9 &\le \sqrt{\frac{1}{4^2} \sum_{i=1}^{n} \ucnormf{\mxitu - \mxitl \mqtl}^2} = \frac{1}{4} \ucnormf{\mxtu - \mxtl \mqtl} \\
	&\le \frac{1}{4} \ucnormf{\mxtu - \mxt} + \frac{1}{4} \ucnormf{\mxt - \mxtl \mqtl} \\
	&\le \frac{1}{4} \etf + \frac{1}{4} \ucdf{\mxt, \mxtl} 
\end{align*}
for each \(l \in [n]\). 

\item Combining all the previous bounds, we deduce that 
\begin{align*}
	\ucnormf{\mft - \mftl \mqtl} \le \sqrtnd \et + I_6 + I_7 + I_8 + I_9 \le \frac 12 \ucdf{\mxt, \mxtl} + 14 \etf + 4 c_0 \ep 
\end{align*}
with probability at least \(1 - \exp(-nd/2) - O(n^{-20})\). 
\end{itemize}

It remains to show \eqref{claim:sigma:frac} that 
\begin{align*}
	\frac{2}{\sigma_d + \sigma_{d}^{(l)}} &\le \frac 32 .  
\end{align*}
By \eqref{induct:xti}, \eqref{eq:1n git}, \eqref{eq:i8} and Weyl's inequality, we have 
\begin{align*}
	\abs{\sigma_{\min} (\mfit) - \sigma_{\min} (\mqt)} &\le \ucnorms{\mfit - \mqt} \le \ucnormf{\mfit - \mqt} \le \ucnormf{\mxit - \mqt} + \frac{1}{n} \ucnormf{\optit(\mgit)} \\
	&\le \ucnormf{\mxit - \mqt} + \frac{1}{n} \ucnormf{\optitu(\mgit)} + \frac{1}{n} \ucnormf{\optit(\mgit) - \optitu(\mgit)} \\
	&\le 3\ep + \frac{1}{n} \ucnormf{\mgit} + I_8 \le \frac 14 , 
\end{align*}
as long as 
\begin{align*}
	\ep \le \frac{1}{32},\  \frac{1}{n} \ucnormf{\mgit} \le \frac{1}{8}, \mbox{ and }\  I_8 \le \frac{1}{4} \etf \le \frac{1}{32}. 
\end{align*}
Additionally, we apply Weyl's inequality once again to deduce that 
\begin{align*}
	\abs{\sigma_{\min} (\mfit) - \sigma_{\min} (\mfitl \mqt)} &\le \ucnormf{\mfit - \mfitl \mqtl} \le \ucnormf{\mft - \mftl \mqtl} \\
	&\le \frac 12 \ucdf{\mxt, \mxtl} + 14 \etf + 4 c_0 \ep \\
	&\le \frac 12 \cdot 2 \ep + 14 \etf + 4 c_0 \ep \le \frac{1}{8}
\end{align*}
due to \eqref{claim:delta f fl}, \eqref{induct:xtl}, and 
\begin{align*}
	\ep \le \frac{1}{32},\ \etf \le \frac{1}{224}, \mbox{ and }\ c_0 \le \frac{1}{4}. 
\end{align*}
As a result, one can ensure that 
\begin{align*}
	\sigma_d = \min_{1 \le i \le n} \sigma_{\min}(\mfit) \ge \frac{3}{4},\quad \sigma_{d}^{(l)} = \sigma_{\min} (\mfitl \mqt) \ge \frac{3}{4} - \frac{1}{8} = \frac{5}{8},
\end{align*}
which implies 
\begin{align*}
	\frac{2}{\sigma_d + \sigma_{d}^{(l)}} &\le \frac 32 .  
\end{align*}

Finally, combine \eqref{claim:delta f fl} and \eqref{claim:sigma:frac} to reach 
\begin{align*}
	\ucdf{\mxtt, \mxttl} &\le \ettf + \frac{2}{\sigma_d + \sigma_{d}^{(l)}} \ucnormf{\mft - \mftl \mqtl} \\
	&\le \ettf + \frac{3}{2} \xkh{\frac 12 \ucdf{\mxt, \mxtl} + 14 \etf + 4 c_0 \ep} \\
	&\le \frac 34 \ucdf{\mxt, \mxtl} + \ettf + 21 \etf + 6 c_0 \ep \\
	&\le 2 \ep 
\end{align*}
for each \(l \in [n]\) as long as 
\begin{align*}
	\ep \le \frac{1}{32},\ c_0 \le \frac{1}{24}, \mbox{ and }\ \etf, \ettf \le \frac{\ep}{88}, 
\end{align*}
which finishes the proof. 

\bibliography{BibforSIAM}

\end{document}